\pdfoutput=1 

\documentclass{article}

\usepackage{microtype}
\usepackage{tabularx}
\usepackage{graphicx}
\usepackage{subcaption}
\usepackage{booktabs} 
\usepackage[hyphens]{url}

\usepackage{tcolorbox}
\newcommand{\mymathhl}[1]{\colorbox{gray!15}{$\displaystyle #1$}}
\newcommand{\mymathinlinehl}[1]{\colorbox{gray!10}{$#1$}}
\usepackage{hyperref}
\usepackage{float}
\usepackage[accepted]{icml2023}

\usepackage{amsmath}
\usepackage{amssymb}
\usepackage{mathtools}
\usepackage{amsthm}
\usepackage{booktabs}
\usepackage{enumitem}
\usepackage[font=small,skip=0pt]{caption}
\newlist{conditions}{enumerate}{1}
\setlist[conditions, 1]{label = Condition \arabic*:}

\newcommand{\minus}{\scalebox{0.75}[1.0]{$-$}}

\usepackage{xpatch}
\makeatletter
\xpatchcmd{\proof}{\topsep6\p@\@plus6\p@\relax}{}{}{}
\makeatother

\makeatletter
\xpatchcmd{\example}{\topsep6\p@\@plus6\p@\relax}{}{}{}
\makeatother

\usepackage[vlined, linesnumbered,ruled,boxed,lined,noend]{algorithm2e}
\SetKwInOut{Parameter}{parameter}
\usepackage{setspace}

\SetCommentSty{mycommfont}

\SetKwComment{/***}{ ***/}{
}

\makeatletter 
\g@addto@macro{\@algocf@init}{\SetKwInOut{Parameter}{Learnable Parameters}} 
\makeatother

\usepackage[capitalize,noabbrev]{cleveref}

\usepackage{color}
\definecolor{myblue2}{rgb}{0.1,0.1,0.7}
\definecolor{myblue}{rgb}{0.1,0.5,0.3}

\newcommand{\torevise}[1]{{\color{gray}\begin{comment} {#1} \end{comment}}}
\newcommand{\revised}[1]{{#1}}

\theoremstyle{plain}
\newtheorem{theorem}{Theorem}[section]
\newtheorem{lemma}[theorem]{Lemma}
\newtheorem{corollary}[theorem]{Corollary}
\theoremstyle{definition}
\newtheorem{proposition}[theorem]{Proposition}
\newtheorem{definition}[theorem]{Definition}

\theoremstyle{remark}
\newtheorem{remark}[theorem]{Remark}
\newtheorem{example}[theorem]{Example}

\usepackage{colortbl}
\newcommand{\bestcell}{\cellcolor{blue!25}}

\usepackage{listings}
\usepackage{etoolbox}
\makeatletter
\AfterEndEnvironment{algorithm}{\let\@algcomment\relax}
\AtEndEnvironment{algorithm}{\kern2pt\hrule\relax\vskip3pt\@algcomment}
\let\@algcomment\relax
\newcommand\algcomment[1]{\def\@algcomment{\footnotesize#1}}
\renewcommand\fs@ruled{\def\@fs@cfont{\bfseries}\let\@fs@capt\floatc@ruled
  \def\@fs@pre{\hrule height.8pt depth0pt \kern2pt}%
  \def\@fs@post{}%
  \def\@fs@mid{\kern2pt\hrule\kern2pt}%
  \let\@fs@iftopcapt\iftrue}
\makeatother
\newcommand{\tmstrong}[1]{\textbf{#1}}
\newcommand{\assign}{:=}

\newcommand{\mathd}{\mathrm{d}}
\newcommand{\nobracket}{}

\newcommand{\tmem}[1]{{\em #1\/}}
\newcommand{\tmop}[1]{\ensuremath{\operatorname{#1}}}

\newenvironment{tmparmod}[3]{\begin{list}{}{\setlength{\topsep}{0pt}\setlength{\leftmargin}{#1}\setlength{\rightmargin}{#2}\setlength{\parindent}{#3}\setlength{\listparindent}{\parindent}\setlength{\itemindent}{\parindent}\setlength{\parsep}{\parskip}} \item[]}{\end{list}}
 
\newcommand{\tmfloatcontents}{}
\newlength{\tmfloatwidth}
\newcommand{\tmfloat}[5]{
  \renewcommand{\tmfloatcontents}{#4}
  \setlength{\tmfloatwidth}{\widthof{\tmfloatcontents}+1in}
  \ifthenelse{\equal{#2}{small}}
    {\setlength{\tmfloatwidth}{0.45\linewidth}}
    {\setlength{\tmfloatwidth}{\linewidth}}
  \begin{minipage}[#1]{\tmfloatwidth}
    \begin{center}
      \tmfloatcontents
      \captionof{#3}{#5}
    \end{center}
  \end{minipage}}
\providecommand{\xequal}[2][]{\mathop{=}\limits_{#1}^{#2}}
\icmltitlerunning{Graph Neural Networks with Learnable and Optimal Polynomial Bases}

\begin{document}

\twocolumn[
\icmltitle{Graph Neural Networks with Learnable and Optimal Polynomial Bases}
\icmlsetsymbol{equal}{*}
\begin{icmlauthorlist}
\icmlauthor{Yuhe Guo}{gaoling}
\icmlauthor{Zhewei Wei}{gaoling,pc,key,moe}
\end{icmlauthorlist}
\icmlaffiliation{gaoling}{Gaoling School of Articial Intelligence, Renmin University of China}
% The work was partially done at Gaoling School of Artificial Intelligence, Peng Cheng Laboratory, Beijing Key Laboratory of Big Data Management and Analysis Methods and MOE Key Lab of Data Engineering and Knowledge Engineering. 
\icmlaffiliation{pc}{Peng Cheng Laboratory}
\icmlaffiliation{key}{Beijing Key Laboratory of Big Data Management and Analysis Methods}
\icmlaffiliation{moe}{MOE Key Lab of Data Engineering and Knowledge Engineering}
\icmlcorrespondingauthor{Zhewei Wei}{zhewei@ruc.edu.cn}
\icmlkeywords{polynomial filter; learnable filter; spectral graph neural networks}
\vskip 0.3in
]

\printAffiliationsAndNotice{}  % leave blank if no need to mention equal contribution

\begin{abstract}
    Polynomial filters, a kind of Graph Neural Networks, typically use a predetermined polynomial basis and learn the coefficients from the training data. It has been observed that the effectiveness of the model is highly dependent on the property of the polynomial basis. Consequently, two natural and fundamental questions arise: Can we learn a suitable polynomial basis from the training data? Can we determine the optimal polynomial basis for a given graph and node features?

    In this paper, we propose two spectral GNN models that provide positive answers to the questions posed above. First, inspired by Favard's Theorem, we propose the FavardGNN model, which learns a polynomial basis from the space of all possible orthonormal bases. Second, we examine the supposedly unsolvable definition of optimal polynomial basis from \citet{Wang2022jacobi} and propose a simple model, OptBasisGNN, which computes the optimal basis for a given graph structure and graph signal. Extensive experiments are conducted to demonstrate the effectiveness of our proposed models.
    Our code is available at  
    \href{https://github.com/yuziGuo/FarOptBasis}{https://github.com/yuziGuo/FarOptBasis}.
\end{abstract}
\section{Introduction}
% \vspace{-1mm}
% Recent years have witnessed the rise of GNNs ...

%  Spectral GNN
\torevise{Spectral Graph Neural Networks are a type of 
Graph Neural Networks that apply filtering operations on graph Laplacian spectrums. }
\revised{
Spectral Graph Neural Networks are a type of Graph Neural Networks 
that comprise the majority of filter-based GNNs~\cite{Shuman2013,isufi2021edgenets,isufi2022graph}.}
They are designed to create graph signal filters in the spectral domain.
% To avoid eigen-decomposition, 
% the desired filtering operations are often approximated by polynomials of laplacian eigenvalues. 
To avoid eigendecomposition, spectral GNNs approximate the desired filtering operations by polynomials of laplacian eigenvalues.

As categorized in \citet{he2022chebii}, 
there are mainly two kinds of spectral GNNs. 
In some works, the desired polynomial filters 
% $h^{\ast}(\cdot)$ 
are \textbf{predefined}.  
For example, 
GCN~\cite{kipf2016semi} fixes 
% $h^{\ast}(\cdot)$ 
the filter 
to be $I - \hat{L}$, 
and APPNP~\cite{klicpera2019appnp} restricts
% $h^{\ast}(\cdot)$ 
the filtering function 
within the Personalized Pagerank.

% \memo{Refine the expression according to Review Qkiv.}

% Another line of research approximates \textbf{arbitrary} 
% filters.
% These works typically allow learnable coefficients upon different truncated polynomial series/bases, 
% and mainly differ in the choice of polynomial bases.  
Another line of research approximates \textbf{arbitrary} filters with learnable polynomials. These models typically fix a predetermined polynomial basis and learn the coefficients from the training data. 
ChebNet~\cite{Defferrard2016cheb} uses Chebyshev basis following the tradition of Graph Signal Processing~\cite{Hammond2009}. 
GPR-GNN~\cite{chien2021gprgnn} uses Monomial basis, which is straightforward.
BernNet~\cite{He2021bern} uses the non-negative Bernstein basis for regularization and interpretation.
JacobiConv~\cite{Wang2022jacobi} chooses among the family of Jacobi polynomial bases, 
with the exact basis determined by two extra hyperparameters. 
% $\alpha$ and $\beta$. 
ChebNetII~\cite{he2022chebii} revisits the Chebyshev basis, 
and incorporates the power of Chebyshev interpolation
 by reparameterizing learnable coefficients by chebynodes.   
Please refer to Section \ref{sec:background_spectral} for more concrete backgrounds about polynomial filtering.   

% Challenges
However, there are still two fundamental challenges on the choice of basis. 

\textbf{Challenge 1: }
\revised{
It is well known and checked by ablation studies~\cite{Wang2022jacobi} that the choice of basis has a significant impact on practical performance.
}
However, the proportion of known polynomial bases is small and may not include the best-fitting basis for a given graph and signal. Therefore, we pose the following question: \textbf{Can we learn a polynomial basis from the training data out of all possible orthonormal polynomials?}
\footnote{For the concrete definition of orthonormal polynomial bases, 
please check the preliminaries in Section \ref{para:inner_product}.}

% \memo{According to Review Qkiv, 
% this empirical motivation is not enough. 
% }

% \textbf{Challenge 1: }
% Up to now, the works that learn filters use \textit{learnable} coefficients 
% above \textit{fixed} polynomial bases, 
% and these different bases show different empirical performances.  
% \textcolor{blue}{Now that the choice of basis matters,  
% and the proportion of polynomial bases that we can name is small, }   
% {can we learn polynomial bases from a wider range}? 
% One step further, 
% to avoid (1) the mutual influence of basis polynomials and (2) the influence of the lens of basis polynomials, 
% we pose a question: \textbf{can we learn arbitrary orthonormal polynomial bases}?
% \footnote{For the concrete definition of orthonormal polynomial bases, 
% please check the preliminaries in Section \ref{para:inner_product}.}

\begin{figure*}
    \centering
    \includegraphics[width=2\columnwidth]{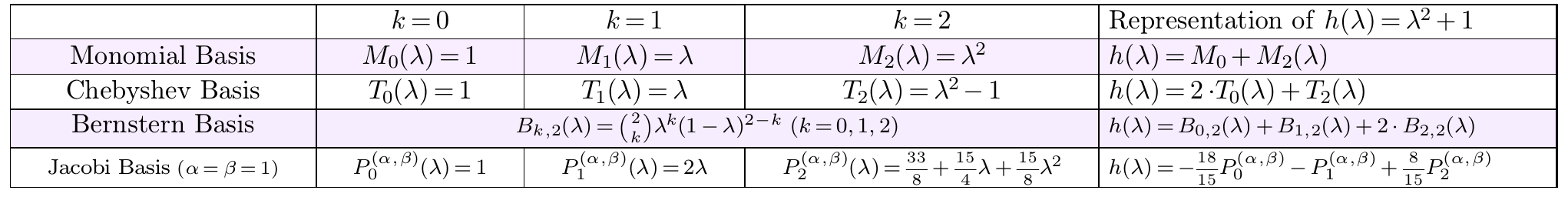}
    \caption{Representation of $h(\lambda)=\lambda^2 + 1$ by 
       different bases.}
    \label{fig:example}
 \end{figure*}

\textbf{Challenge 2: }
On the other hand, although these bases differ in empirical performances, their expressiveness should be the same:
any target polynomial of order $K$ can be represented by 
any complete polynomial basis with truncated order $K$ (See Figure \ref{fig:example} for an example). 
% Another question naturally rises: where does the discrepancy in empirical performances come from? 
% \textbf{Is there a criterion to measure the quality of basis?} 
Therefore, \citet{Wang2022jacobi} raised a definition of \textit{optimal basis} from an optimization perspective,   
which promises an optimal convergence rate. 
However, this basis is believed to be unsolvable using existing techniques. Consequently, a natural question is: \textbf{can we compute this optimal basis for a given graph and signal using innovative techniques?}
% however, they consider this optimal basis to be unsolvable. 
% \textbf{
% Can we utilize the optimal basis?}

% \begin{figure*}
%     \centering
%     \includegraphics[width=2\columnwidth]{images/example.pdf}
%     \caption{Representation of $h(\lambda)=\lambda^2 + 1$ by 
%        different bases.}
%     \label{fig:example}
%  \end{figure*}

% Contribution
% To tackle the first challenge, 
% we leverage two Theorems in orthogonal polynomials: Three-term recurrences and Favard's Theorem.  

In this paper, we provide positive answers to the questions
posed above. We summarize our contributions in three folds. 
Firstly, we propose FavardGNN with  \textbf{learnable orthonormal basis} to tackle the first challenge. The theoretical basis of 
FavardGNN is two Theorems in orthogonal polynomials: the Three-term recurrences and its converse, Favard's Theorem. FavardGNN learns from the \textit{whole space} of possible orthonormal basis with $2(K+1)$ extra parameters.
Secondly, we propose OptBasisGNN with \textbf{solvable optimal basis}. 
We solve the optimal basis raised by \citet{Wang2022jacobi} by avoiding the explicit solving of the weight function 
\revised{, which invites the need for eigendecomposition}. 
Note that although we write out the implicitly defined/solved polynomial series in the methodology section, 
% and show that it is a special case of FavardGNN, 
we never need to solve it explicitly.   
Last but not least, we conduct \textbf{extensive experiments}
to 
% , including (1) node classification tasks that are scaled up to ogbn-Papers100M dataset, (2) multi-channel filter learning tasks to dive into the convergence property of bases, (3) comparing between OptBasisGNN and FavardGNN. The experiments 
demonstrate the effectiveness of our proposed models.

\section{Background and Preliminaries}
\subsection{Background of Spectral GNNs}
\label{sec:background_spectral}
In this section, 
we provide some necessary backgrounds of spectral graph neural networks, 
% The idea of spectral gnns mainly starts from using graph filtering without eigen-decomposition. The choice of 
and show how the choice of polynomial bases emerges as a problem. 
Notations used are summarized in 
Table \ref{tbl:notations} in 
Appendix \ref{sec:notations}.

\textbf{Graph Fourier Transform.\quad}
Consider an undirected and connected graph $G = (V, E)$ with $N$ nodes, 
% after being processed to be 
its symmetric normalized adjacency matrix 
and laplacian matrix are denoted as
$\hat{P}$ and $\hat{L}$, respectively, $\hat{L}={I}-\hat{P}$.
% Researchers in graph signal processing found 
% analogous concepts as in \textit{classic Fourier Transform} 
\textit{Graph Fourier Transform}, as defined in the spatial/spectral domain of graph signal processing, is analogous to the time/frequency domain Fourier Transform
\cite{Hammond2009, Shuman2013}
. 
One column in the representations of $N$ nodes, 
$X \in \mathbb{R}^{N \times d}$, 
is considered a \textit{graph signal}, 
denoted as $x$.
The complete set of $N$ eigenvectors of $\hat{L}$, denoted as $U$,
who show varying structural frequency characteristics~\cite{Shuman2013},
are used as \textit{frequency components}.
\textit{Graph Fourier Transform} 
is defined as $\mymathinlinehl{\hat{x}\assign U^\mathrm{T}x}$,
where signal $x$ is projected to the frequency responses of all components.
It is then followed by 
\textit{modulation},
which suppresses or strengthens certain frequency components, 
denoted as  
$\mymathinlinehl{\hat{x}^{*} \assign \tmop{diag} \{ \theta_0, \cdots, \theta_{N-1} \}\hat{x}}$.
After modulation, \textit{inverse Fourier Transform}:
$\mymathinlinehl{x^{*}\assign U\hat{x}^{*}}$ transforms $\hat{x}^{*}$
back to the spatial domain.
The three operations form the process of \textit{spectral filtering}:
% $a+1=b~\refstepcounter{equation}(\theequation)\label{myeq}$
$\mymathinlinehl{U \tmop{diag} \{ \theta_0, \theta_1, \ldots, \theta_{N-1} \} U^\mathrm{T} x
~\refstepcounter{equation}(\theequation)\label{eq-IFT}}$.

\textbf{Polynomial Approximated Filtering.\quad}
% To avoid exact eigen-decomposition, 
In order to avoid time-consuming eigendecomposition,
a line of work approximate  
$\theta_i$ by 
some polynomial function of $\lambda_i$, 
which is the $i$-th eigenvalue of $\hat{L}$, 
% with a given polynomial basis $\{ g_k (\cdot) \}_{k = 1}^K$, 
i.e. 
$\theta_i \approx h(\lambda_i)$. 
Equation~\eqref{eq-IFT} then becomes a form that is easy for fast
{\textit{localized}} calculation:
$
    U \tmop{diag} \{ h (\lambda_0), h (\lambda_1), \dots, h(\lambda_{N-1}) \} U^\mathrm{T} x = h (\hat{L}) x .
$
As listed in Introduction, 
various \textit{polynomial bases} have been utilized, 
denoted as ${h(\lambda) = \sum^K_{k = 0} \alpha_k g_k (\lambda)}$. 
For further simplicity, 
we equivalently use $b (\hat{P})$ instead of $h(\hat{L})$ in this paper, where 
% \mymathinlinehl{ b (\hat{P}) := h (I - \hat{P}) }
$b (\hat{P}) := h (I - \hat{P})$. 
Note that $b (\cdot)$ is defined on the spectrum of $\hat{P}$, 
and the $i$-th eigenvalue of $\hat{P}$, denoted as $\mu_i$, equals $1-\lambda_i$.

The filtering process on the input signal $x$ is then expressed as 
$\mymathinlinehl{x \rightarrow z = \sum_{k = 0}^K \alpha_k g_k (\hat{P}) x}$.
When consider 
independent filtering on each of the $d$ channels in $X$ simultaneously, the \textbf{multichannel filtering} can be denoted as:   
$
X \rightarrow Z = \underset{l \in [1, h]}{\|}  \sum_{k = 0}^K \alpha_{k, l}
  g_{k, l} (\hat{P}) X_{:, l}
~\refstepcounter{equation}(\theequation)
\label{eq:MultiChannelFilter}
$.

\subsection{Orthogonal and Orthonormal  Polynomials}

In this section, we give a formal definition of orthogonal and orthonormal polynomials, which plays a central role in the choosing of polynomial bases~\cite{simon2014spectral}.

% This section is mainly about orthogonal polynomials \cite{simon2014spectral}. 

\textbf{Inner Products.}
% \citep[~p.82]{mason2002chebyshev}.
\label{para:inner_product}
The inner product of polynomials 
% $f$ and $g$ 
is defined as 
$\mymathinlinehl{ \langle f, g \rangle := \int_{a}^{b} f(x) g(x) w(x)  \mathd x  }
$,
% \[
%     \langle f, g \rangle := 
%     \int_{a}^{b} f(x) g(x) w(x)  \mathd x , 
% \] 
where $f$, $g$ and $w$ are functions of $x$ on interval $(a,b)$, 
and the \textit{weight function} $w$ should be non-negative 
to 
guarantee the 
positive-definiteness of inner-product space.

% non-negativity of inner-product $\langle f, f \rangle$.  

The definition of the inner products induces 
the definitions of \textit{norm} and \textit{orthogonality}.
The norm of polynomial $f$ is defined as: 
$\mymathhl{\|f\| = \sqrt{\langle f, f \rangle}}$,
% \[
%     \|f\| = \sqrt{\langle f, f \rangle} ,
% \]
and $f$ and $g$ are orthogonal to each other when 
$\mymathhl{\langle f, g \rangle = 0}$. Notice that the concept of inner product, norm, and orthogonality
are all defined with respect to some weight function. 

\textbf{Orthogonal Polynomials.}
% \citep[~p.83]{mason2002chebyshev}.
A sequence of polynomials $\{p_n(x)\}_{n=0}^{\infty}$ 
where $p_n(x)$ is of exact degree $n$, is called \textit{orthogonal} 
w.r.t. the positive weight function $w(x)$ if, 
for $m, n=0,1,2,\cdots$, there exists  
$\langle p_n, p_m \rangle = \delta_{mn} \|p_n\|^2 (\|p_n\|^2 \neq 0)$,
where the inner product $\langle f, g \rangle$ is defined w.r.t.
$w(x)$. When $\|p_n\|^2 = 1$ for $n=0,1,2,\cdots$, $\{p_n(x)\}_{n=0}^{\infty}$ 
is known as \textbf{orthonormal} polynomial series.

When a weight function is given, the orthogonal or orthonormal series 
% of $\{p_n(x)\}_{n=0}^{\infty}$ 
with respect to the weight function 
can be solved by \textit{Gram-Schmidt process}.

% \textbf{\textit{Remark}.}
\begin{remark}
\label{remark:increasing_order}
In this paper, the orthogonal/orthonormal polynomial bases we consider 
are truncated polynomial series, i.e. the polynomials that form a basis are of 
increasing order. 
\end{remark}

\section{Learnable Basis via Favard's Theorem}
\label{sec:methodI}
% \vspace{-0.5mm}

% \subsection{Recurrence Formula for Orthogonal Series}

% In this section, we give a brief introduction to two theorems in orthogonal polynomials: 
% the {three term recurrences} and a converse theorem: Favard's theorem.  
% The concrete form and proofs of the theorems are put in Appendix.

Empirically,
spectral GNNs with different polynomial bases vary in performance on different datasets, which leads to two observations: (1) the choice of bases matters; (2) whether a basis is preferred might be related to the input, i.e. 
different signals on their accompanying underlying graphs.

For the first observation, we notice that up to now, polynomial filters \textit{pick} polynomial bases 
from well-studied polynomials, e.g. Chebyshev polynomials, Bernstein polynomials, \textit{etc}, 
which narrows down the range of choice.  
For the second observation, we question the reasonableness of fixing a basis during training. 
A related effort is made by JacobiConv~\cite{wang2019improving}, 
who adapt to a Jacobi polynomial series from the family of Jacobi polynomials via \textit{hyperparameter tuning}. However, 
the range they choose from is discrete. 
Therefore, we aim at dynamically \textbf{learn} polynomial basis from the input from a \textbf{vast range}.

% \vspace{-1mm}
% \subsection{Three-Term Recurrences and Favard's Theorem} 
% \vspace{-1mm}
% \input{submodules/three_term_briefly.tex}

% We extend these two theorems to orthonormal polynomials, 
% which form the fundamentals of our first proposed model: FavardGNN. 
% FavardGNN can learn any possible orthonormal basis by $2(K+1)$ learnable parameters. 

% By Favard's Theorem, 
% we can push the utilization of three-term recurrences in polynomial filters from several special polynomials (i.e. Chebyshev basis, Jacobi Basis) to 
% any orthogonal polynomials, 
% by simply learn 
% any possible recurrence coefficients. 

% \vspace{-1mm}
\subsection{Recurrence Formula for Orthonormal Bases}
% \vspace{-1mm}
\begin{algorithm}[tb]
    \SetAlgoNoLine
    \KwIn{Input signals $X$ with $d$ channels; 
            Normalized graph adjacency $\hat{P}$; 
            Truncated polynomial order $K$
            }
    \Parameter{$\beta$, $\gamma$, $\alpha$
          }
    \KwOut{Filtered Signals $Z$} 
    \BlankLine
    % \Indp
    % \Foreach{channel \in $X$}{}
    $x_{-1} \leftarrow 0$\\
    \For{$l=0$ \KwTo $d-1$}{ 
        $x \leftarrow X_{:, l}$
        ,\ $x_0 \leftarrow x / \sqrt{\beta_{0, l}}$
        ,\  $z \leftarrow \alpha_{0, l} x_0$
        \\ \For {$k=0$ \KwTo $K$}{
            $x_{k + 1} \leftarrow (\nobracket \hat{P} x_k - \gamma_{k,l} x_k - \sqrt{\beta_{k, l}} x_{k - 1}) / \sqrt{\beta_{k + 1, l}}$
            \\ $z \leftarrow z + \alpha_{k + 1, l} x_{k + 1}$
        }
        $Z_{:,l} \leftarrow z$
    }
    \KwRet{Z}
    \caption{\textsc{FavardFiltering}}
    \label{alg:favard}
\end{algorithm}
\begin{algorithm}[t]
    % \setstretch{0.9}
    \KwIn{Raw features $X_\textrm{raw}$; 
            Normalized graph adjacency $\hat{P}$; 
            Truncated polynomial order $K$
            }
    \Parameter{$W_0$, $b_0$, $W_1$, $b_1$,
          $\beta$, $\gamma$, $\alpha$
          }
    \KwOut{Label predictions $\hat{Y}$} 
    \BlankLine\Indp
    $X \leftarrow X_{\tmop{raw}} W_0 + b_0$
    \\
    $Z \leftarrow${\textsc{FavardFiltering}}($X$,
      $\hat{P}$, $K$, $\beta$, $\gamma$, $\alpha$)
    \\
    $\hat{Y} \leftarrow$Softmax($Z W_1 + b_1$)
    \caption{\textsc{FavardGNN} (For Classification)}
    \label{alg:favardgnn_cls}
\end{algorithm}
Luckily, the Three-term recurrences and Favard's theorem of orthonormal polynomials provide a 
\textit{continuous} parameter space to learn 
 basis.   Generally speaking, three-term recurrences states that every orthonormal polynomial series satisfies a very characteristic form of recurrence relation, and   
Favard's theorem states the converse.

%We investigate the orthonormal form of three-term recurrences and Favard's theorem. 

% These two theorems can be found in early chapters of textbooks or monographs about orthogonal polynomials~\cite{gautschi2004orthogonal, simon2014spectral}, 
% % and the proofs can be found in  
% we provide some proofs footed on a minimum background in Appendix \ref{sec:proof-of-3term} to \ref{sec:proof-of-favard-orthonormal} for the convenience to check through.

% From the Favard's theorem for general case (Corollary \ref{thm:far}) 
% and the three term recurrence relations for orthogonal series (Theorem \ref{thm:3term}), 
% we arrive at a point: 
% all possible recurrences given by Corollary \ref{thm:far} 
% define all possible orthogonal basis
% \footnote{Notice remark \ref{remark:increasing_order}}. 
% This gives us a wider range to choose polynomial basis than Jacobi
% polynomials. In fact, Jacobi polynomials are special cases of Corollary

\begin{theorem}[Three Term Recurrences for Orthonormal Polynomials] 
\citep[p.~12]{gautschi2004orthogonal}
% \citealp[p.~12]{gautschi2004orthogonal}
% \cite{gautschi2004orthogonal}
    \label{thm:3term_orthonormal}
    For orthonormal polynomials $\{ p_k \}_{k=0}^{\infty}$ w.r.t. weight function $w$, 
    suppose that the leading coefficients of all polynomials are positive, 
    there exists the three-term recurrence relation:
    \begin{align}
    \label{eq:formula_orthonormal}
    \mymathinlinehl{ \sqrt{\beta_{k + 1}}} p_{k + 1} (x) 
    & = (x - \gamma_k) p_k (x) -
            \mymathinlinehl{\sqrt{\beta_k}} p_{k - 1} (x), 
    \notag \\
    & p_{-1}(x) \assign 0, \ p_0 (x) = 1 / \sqrt{\beta_0}, 
    \notag \\
    & \gamma_k \in \mathbb{R}, \ \sqrt{\beta_k} \in \mathbb{R}^{+}, \  k \geq 0
\end{align}
with $\beta_0 = \int w(x) \mathd x$.

\end{theorem}

% \begin{proof}
%   See Appendix \ref{sec:proof-of-3term-orthonormal}.
% \end{proof}

% \begin{corollary}[Favard's Theorem; Orthonormal case]
%     \label{thm:far-orthonormal}
%     If a sequence of polynomials $\{ p_k \}_{k = 0}^{\infty}$ statisfies a
%     three-term recurrence relation
%     \[  \sqrt{\beta_{k + 1}} p_{k + 1} (x) = (x - \gamma_k) p_k (x) -
%        \sqrt{\beta_k} p_{k - 1} (x), \ k = 0, 1, \cdots \]
%     and 
%     \[
%       p_{- 1} (x) \equiv 0, \ p_0 (x) \equiv 1 / \sqrt{\beta_0}, 
%     \]
%     with
%     $\gamma_k \in \mathbb{R}$ and $\sqrt{\beta_k} \in \mathbb{R}^+$,
%     then there exists a positive weight function $w$
%     such that $\{ p_k \}_{k = 0}^{\infty}$ is orthonormal with respect to 
%     $w$, and $\beta_0 = \int_a^b w(x) \mathd x$.
% \end{corollary}

\begin{theorem}[Favard's Theorem; Orthonormal Case]
\cite{favard1935polynomes}, \citep[p.~14]{simon54orthogonal}
    \label{thm:far-orthonormal}
    A polynomial series $\{ p_k \}_{k = 0}^{\infty}$ who satisfies the recurrence relation in Equation~\eqref{eq:formula_orthonormal}
    is orthonormal w.r.t. a weight function $w$ that  $\beta_0 = \int w(x) \mathd x$. 
\end{theorem}

% \begin{proof}
%   See Appendix \ref{sec:proof-of-favard-orthonormal}.
% \end{proof}

% \begin{example}
%   See Appendix xx.
% \end{example}

% \paragraph*{Conclusion.}
By Theorem \ref{thm:far-orthonormal}, 
any possible recurrences with the form~\eqref{eq:formula_orthonormal} 
defines an orthonormal basis. 
By Theorem~\ref{thm:3term_orthonormal}, 
such a formula covers the whole space of orthonormal polynomials.
If we set $\{\sqrt{\beta_k}\}$ and $\{\gamma_k\}$ to be learnable parameters with $\sqrt{\beta_k}>0(k\geq 0)$, any orthonormal basis can be obtained.

% \textbf{More General theorems and }
We put the more general \textit{orthogonal} form of Theorem~\ref{thm:3term_orthonormal} and Theorem~\ref{thm:far-orthonormal} 
% , and their relations with the orthonormal form 
in Appendix~\ref{sec:proof-of-3term} to \ref{sec:proof-of-favard-orthonormal}.  
% For the convenience of interested readers to check through, we also provided some related proofs, although the theorems 
% can be found in early chapters monographs about orthogonal polynomials~\cite{gautschi2004orthogonal, simon2014spectral}. 
% % and the proofs can be found in  
% we provide some proofs footed on a minimum background in Appendix \ref{sec:proof-of-3term} to \ref{sec:proof-of-favard-orthonormal} for the convenience to check through.
In fact, the property of three-term recurrences for orthogonal polynomials has been used multiple times 
in the context of current spectral GNNs
to reuse $g_k(\hat{P})x$ and $g_{k-1}(\hat{P})x$ 
for the calculation of $g_{k+1}(\hat{P})x$.  
% In polynomial filters, such recurrence relation permits the calculation of $g_{k+1}(\hat{P})x$ to reuse $g_k(\hat{P})x$ and $g_{k-1}(\hat{P})x$. 
\citet{Defferrard2016cheb} 
owe the fast filtering of ChebNet to employing the three-term recurrences of \textit{Chebyshev polynomials} (the first kind, 
which is orthogonal w.r.t.  $\frac{1}{\sqrt{x^2-1}}$): 
$T_{k+1}(x) = 2xT_k(x)-T_{k-1}(x)$. 
Similarly, 
JacobiConv~\cite{Wang2022jacobi} employs the 
three-term recurrences for \textit{Jacobi polynomials} (orthogonal w.r.t. to $(1-x)^{a}(1+x)^{b}$). In this paper, however, we focus on orthonormal bases because they minimize  the mutual influence of basis polynomials and the influence of the unequal norms of different basis polynomials.

\subsection{FavardGNN}

\textbf{Formulation of FavardGNN.\quad}We formally write the architecture of
\textsc{FavardGNN} (Algorithm \ref{alg:favardgnn_cls}), with the filtering process illustrated in \textsc{FavardFiltering} (Algorithm \ref{alg:favard}). Note that the iterative process of Algorithm \ref{alg:favard} (lines 3-5) follows exactly from Equation~\eqref{eq:formula_orthonormal} in Favard's Theorem. The key insight is to treat the coefficients $\beta, \gamma, \alpha$ in Equation~\eqref{eq:formula_orthonormal} as learnable parameters. Since Theorem~\ref{thm:3term_orthonormal} and Theorem~\ref{thm:far-orthonormal}  state that the orthonormal basis must satisfy the employed iteration and vice versa, it follows that the model can learn a suitable orthonormal polynomial basis from among all possible orthonormal bases.

Following convention, before \textsc{FavardFiltering}, an MLP is used to map the raw features onto the signal
channels (often much less than the dimension of raw features). In regression
problems, the filtered signals are directly used as predictions; for
classification problems, they are combined by another MLP followed by a
softmax layer.

\textbf{Parallel Execution.\quad}Note that for convenience of presentation, we
write the \textsc{FavardFiltering} Algorithm in a form of nested loops. In fact, the
computation on different channels (the inner loop $k$) is conducted
simultaneously. We put more concrete implementation in PyTorch-styled
the pseudocode in Appendix \ref{sec:pseudo_torch_Favard}.

% {\color[HTML]{008080}{\noindent}\begin{tmparmod}{0pt}{0pt}{0em}%
%   \begin{tmparsep}{0em}%
%     {\textbf{Algorithm 3 \tmcolor{blue}{(FavardFiltering)} }}{\smallskip}
    
%     \begin{tmindent}
%       {\textbf{Input}}: Input signals $X $; Normalized Graph $\hat{P}$;
%       Order $K$; $\beta$, $\gamma$; coefficients $\alpha$
      
%       {\textbf{Learnable Parameters}}: \ $\sqrt{\beta}$, $\gamma$, $\alpha$
      
%       {\textbf{Output}}: Filtered signals $Z$
      
%       {\textbf{Procedure:}}
      
%       \qquad{\textbf{for}} the $l$-th channel {\textbf{in}} $X$,
%       {\textbf{do}}
      
%       {\hspace{4em}}$x \leftarrow X_{:, l}$
      
%       {\hspace{4em}}$x_0 \leftarrow x / \sqrt{\beta_{0, l}}$
      
%       {\hspace{4em}}$x_{- 1} \leftarrow 0$,
      
%       {\hspace{4em}}$z \leftarrow \alpha_{0, l} x_0$ ,
      
%       {\hspace{4em}}{\textbf{for}} $k$ {\textbf{in}} $[0, K]$,
%       {\textbf{do}}
      
%       {\hspace{6em}}$x_{k + 1} \leftarrow (\nobracket \hat{P} x_k - \gamma_{k,
%       l} x_k - \sqrt{\beta_{k, l}} x_{k - 1}$)/$\sqrt{\beta_{k + 1, l}}$
      
%       {\hspace{6em}}$z \leftarrow z + \alpha_{k + 1, l} x_{k + 1}$
      
%       {\hspace{4em}}$Z_{:, l} \leftarrow z$
      
%       \ 
%     \end{tmindent}
%   \end{tmparsep}
% \end{tmparmod}{\medskip}}

% \input{algorithms/favard.tex}

\subsection{Weaknesses of FavardGNN}
\label{sec:weakness}
% - The range to choose from is too wide.

% - The weight function lacks interpretability / is not possible to know.

% - The problem is not convex, yielding to inferior performance in convergence
% empirically.
However, there are still two main weaknesses of FavardGNN. Firstly, the orthogonality lacks interpretability. 
The weight function $w$ can only be solved analytically in a number of cases \cite{Geronimo1991weightfunction}.  
Even if the weight function is solved, the form of $w$ might be too complicated to understand. 

Secondly, 
\textsc{FavardFiltering} is not good in  convergence properties:
consider a simplified optimization problem $\min \|Z-Y\|^2_\textrm{F}$ which has been examined in the context of GNN 
\cite{keyulu2021Optm, Wang2022jacobi}, 
even this problem is non-convex w.r.t the learnable parameters in $Z$. 
We will re-examine this problem in the experiment section. 
\section{Achieving Optimal Basis}
\label{sec:optbasis}

Although FavardGNN potentially reaches the whole space of orthonormal polynomial series, 
on the other hand, 
we still want to know: 
\textbf{whether there is an optimal and accessible basis} 
in this vast space.

Recently, \citet{Wang2022jacobi} raises a criterion for 
optimal basis. Since different bases are the same in expressiveness,
this criterion is induced from an angle of optimization.
However, \citet{Wang2022jacobi} believe that this optimal basis 
is unreachable.
In this section, we follow this definition of optimal basis, 
and show how we can \textit{exactly} apply this optimal basis to our polynomial filter  
with $O (K | E |)$ time complexity.

\citet{Wang2022jacobi} make an essential step towards this question: 
they derive and define an optimal basis from the angle of optimization. 
However, they do not exhaust their own finding in their model, 
since based on a habitual process, they believe that the optimal basis they find is inaccessible. 
In this section, we show how we can \textit{exactly} apply this optimal basis to our polynomial filter 
in $O (K | E |)$ time complexity. 
% The final model we raise is in Algorithm~\ref{alg:OptBasisFiltering}. 

\subsection{A Review: A Definition for Optimal Basis}
\label{sec:optdefinition}
We start this section with a quick review of the related part from \citet{Wang2022jacobi}, with a more complete review put in Appendix \ref{sec:SumWang}.

\textbf{Definition of Optimal Basis.\quad}
\label{sec:SumWang_short}
\citet{Wang2022jacobi} considers the squared loss 
$R = \frac{1}{2} \| Z - Y \|_\textrm{F}^2$, 
where $Y$ is the target signal.
% and $Z = \underset{l \in [1, h]}{\|}  \sum_{k = 0}^K \alpha_{k, l}
% g_{k, l} (\hat{P}) X_{:, l}$ . \footnote{Here, $X$ is not necessarily the raw feature ($X_{raw}$) but often some thing like $X_{raw}W$. $W$ is irrelavant to the choice of polynomial basis, 
% and merges $W$ into $X$.}
Since each signal channel 
contributes independently to the loss, 
% i.e. $R = \sum_l \frac{1}{2} \| Z_{:, l} - Y_{:, l} \|_{F }^2$, 
the authors then consider the loss function channelwisely and 
ignore the index $l$, that is,
% \[ r = \frac{1}{2} \| z - y \|^2_F, \]
$r = \frac{1}{2} \| z - y \|^2_\textrm{F}$, 
where $z = \sum_{k = 0}^K \alpha_k g_k (\hat{P}) x$.

% \revised{
The task at hand is to seek a polynomial series $\{g_k\}_{k=0}^{K}$ which is \textit{optimal} for the convergence of coefficients $\alpha$.
% }
Since $r$ is convex w.r.t. $\alpha$,
the gradient descent's convergence rate reaches optimal 
when the {\textbf{Hessian matrix}} is identity. 
The $(k_1, k_2)$ element 
$(k_1, k_2 \in \left[0,K\right])$
of the Hessian matrix is:
\begin{equation}
    H_{k_1 k_2} = \frac{\partial^2 r}{\partial \alpha_{k_1}
    \partial \alpha_{k_2}} = x^\mathrm{T}
    g_{k_2} (\hat{P}) g_{k_1} (\hat{P}) x.
    \label{eqHessian}
  \end{equation}

\begin{tcolorbox}[boxrule=0.pt,height=18mm,valign=center,colback=blue!3!white]
    \begin{definition}[Optimal basis for signal $x$]
        For a given graph signal $x$, polynomial basis $\{ g_k \}_{k = 0}^K$ 
        is optimal in convergence rate when $H$
        given in \eqref{eqHessian} is an \textbf{identity matrix}.
    \label{def:opt_basis}
    \end{definition}
\end{tcolorbox}

% \revised{
\citet{Wang2022jacobi} further reveal the orthonormality inherent in the optimal basis by rephrasing Equation~\eqref{eqHessian} into  
$
H_{k_1 k_2} = \int_{\mu = - 1}^1 g_{k_1} (\mu) g_{k_2} (\mu) f
   (\mu) \mathd \mu 
$,
where the form of $f$ is given in Proposition~\ref{prop:exact_weight} and 
and derivation is delayed in Appendix~\ref{sec:SumWang}.
Combining Definition \ref{def:opt_basis},
we soonly get:

\begin{tcolorbox}[boxrule=0.pt,height=19mm,valign=center,colback=blue!3!white]
\begin{proposition}[Exact weight function of optimal basis]
    The optimal polynomial basis in Definition \ref{def:opt_basis} is orthonormal w.r.t. weight function $f$, 
        where 
        $f (\mu) = \frac{^{} \vartriangle F (\mu)}{\vartriangle\mu}$, 
        with
        $F (\mu) \assign \sum_{\mu_i \leq \mu} (U^\mathrm{T} x)_i^2$.
        \label{prop:exact_weight}
\end{proposition}
\end{tcolorbox}
% }

\paragraph{Unachievable Algorithm Towards Optimal Basis.}
% \revised{
Now we illustrate why \citet{Wang2022jacobi} believe that
though properly defined, this optimal basis is unachievable,     
and how they took a step back to get their final model. 
We summarize the process they thought of in Algorithm \ref{alg:unreacheable}. 
This process is quite habitual:  
with the weight function in Proposition~\ref{prop:exact_weight} solved,  
it is natural to use it to determine the first $K$ polynomials by the Gram-Schmidt process 
and then use the solved polynomials in filtering as other bases, e.g. Chebyshev polynomials. 
This process is unreachable due to the eigendecomposition step, 
which is essential for the calculation of $f$ (see Proposition~\ref{prop:exact_weight}), 
but prohibitively expensive for larger graphs. 
% }

\begin{algorithm}[tb]
    \SetAlgoLined
% \begin{algorithmic}[1]
    \KwIn{Graph signal $x$; 
            Normalized graph adjacency $\hat{P}$; 
            Truncated polynomial order $K$
            } 
    \KwOut{Optimal basis ${\{g_k(\cdot)}\}_{k=0}^{K}$}
    \BlankLine
    \Indp 
    {$U, \{ \mu_i \}_{i = 1}^N \leftarrow$} Eigendecomposition of {$\hat{P}$}
    \\
    Calculate $f(\mu)$ as descripted in
    % ~\ref{sec-Wang}
    Proposition~\ref{prop:exact_weight}
    \\
    Use Gram-Schmidt process and weight function $f (\mu)$ to contruct an orthonormal basis $\{ g_k \}_{k = 0}^K$
    \\
    Apply $\{ g_k \}_{k = 0}^K$ in polynomial filtering
    \caption{(An Unreachable Algorithm for Utilizing Optimal Basis)}
    \label{alg:unreacheable}
\end{algorithm}

As a result, 
\citet{Wang2022jacobi} came up with a compromise. They allow their
model, namely JacoviConv, to choose from the family of orthogonal Jacobi bases, 
who have "{\tmem{flexible enough weight functions}}'', 
i.e., $(1 - \mu)^{a} (1 + \mu)^{b} \left(\forall a,b\in \left(0,1\right)  \right)$. 
In their implementation,  $a$ and $b$ are discretized and chosen via hyperparameter tuning. 
% The
% Jacobi bases are a family of polynomial bases. A specific form Jacobi basis is
% determined by two parameters $(\alpha, \beta)$. Similar to the well-known
% Chebyshev basis, Jacobi bases have a recursive formulation, making them
% efficient for calculation.
Obviously, the fraction of possible weight functions JacoviConv can cover is still small, 
very possibly missing the optimal weight function in Proposition~\ref{prop:exact_weight}.

% \subsection{From Polynomial Basis to Vector Basis}

\subsection{OptBasisGNN}

In this section, we show how the polynomial filter 
can employ the optimal basis in Definition~\ref{def:opt_basis} 
efficiently via an innovative $O (K | E |)$ methodology. 
Our method does not follow the convention in  
Algorithm \ref{alg:unreacheable} 
where four progressive steps  
are included to solve the optimal polynomial bases out and utilize them.   
% and thus bypasses the untractable eigendecomposition step. 
Instead, our solution to the optimal bases is implicit, 
accompanying the process of solving a related vector series. 
Thus, our method bypasses the untractable eigendecomposition step.

\textbf{Optimal Vector Basis with Accompanying Polynomials.\quad} 
% Following the analysis in Section~\ref{sec:optdefinition}, 
Still, we consider graph signal filtering on one channel, 
that is,
$x \rightarrow z = \sum_{k = 0}^K \alpha_k g_k (\hat{P}) x$.
Instead of taking the matrix polynomial 
$
% \mymathinlinehl{
b (\hat{P}) = \sum_{k = 0}^K \alpha_k g_k (\hat{P})
% }
$
% $\sum_{k = 0}^K \alpha_k g_k (\hat{P}) = b (\hat{P})$
as a whole, 
we now regard 
$
% \mymathinlinehl{
\{ v_k  | \nobracket v_k \assign g_k (\hat{P}) x \}_{k =0}^K
% }
$
as a \textit{vector basis}. 
Then the filtered signal $z$ is a linear combination of 
% $\{ v_k \}_{k =0}^K$
the vector basis, namely 
$
\mymathinlinehl{
x \rightarrow z = \sum_{k=0}^{K} \alpha_k v_k
~\refstepcounter{equation}(\theequation)
\label{eq:vecbasis}
}
$.
When $\{ g_k \}_{k=0}^{K}$ meets Definition \ref{def:opt_basis}, 
for all $k_1, k_2 \in [0, K]$, 
the vector basis satisfies:
\begin{align}
  v_{k_2}^\mathrm{T} v_{k_1} = x^\mathrm{T} g_{k_2} (\hat{P}) g_{k_1} (\hat{P}) x = \delta_{k_1
   k_2}.
   \label{eq:optvecbasis}
\end{align}

Given $(\hat{P}, x)$, 
we term $g_k$ the \textit{\textbf{accompanying polynomial}} of a vector $v_k$
if $v_k = g_k(\hat{P})x$. 
Note that an accompanying polynomial does not always exists for any vector. 
Following Equation~\eqref{eq:optvecbasis}, finding the optimal \textit{polynomial} basis for filtering  
is equivalent to finding a \textit{vector} basis $\{ v_k \}_{k = 0}^K$ that 
satisfies two conditions: 
\textbf{Condition 1}: Orthonormality; 
\textbf{Condition 2}:  Accompanied by the optimal polynomial basis, that is, 
$v_k \equiv g_k(\hat{P})x$ establishes for each $k$, 
where $g_k$ follows Definition~\ref{def:opt_basis}.
We term such $\{v_k\}$ the optimal vector basis.

\begin{algorithm}[t]
    \SetAlgoNoLine
    \caption{
        \textsc{OptBasisFiltering} \\ 
        {
          \footnotesize 
          \linespread{0.8} 
          \textbf{1.} In the comment, we write the implicitly undergoing process of obtaining the accompanying optimal polynomial basis.
          \\
          \textbf{2.} Steps 1-3 will be further substituted by Algorithm~\ref{alg:nextbasis} after the derivative of Proposition~\ref{prop:onlytwo}. 
          % The $(k+1)$-th optimal basis polynomial \textcolor{blue}{$g_{k+1}(\cdot)$}  
          % based on \textcolor{blue}{$\{g_{k}\}_{k=0}^{K}$} 
          % that are implicitly used but never solved explicitly.
          % For initialization, $g_{-1}=0, g_{0}=1/\|x\|$.
        }
    }
    \KwIn{
        Input signals $X$ with $d$ channels; 
        Normalized graph adjacency $\hat{P}$;
        Order $K$
        }
    \Parameter{$\alpha$}
    \KwOut{Filtered signals $Z$} 
    \BlankLine
    \For{$l=0$ \KwTo $d-1$}{ 
        $x \leftarrow X_{:, l}$ 
        \\ $v_0 \leftarrow x / \| x \|$ \tcp*[r]{$g_{0}(\mu)=1/\|x\|$} 
        % \\ $z \leftarrow \alpha_{0, l} v_0$   
        $z \leftarrow \alpha_{0, l} v_0$  
        \\ \For{$k=0$ \KwTo $K$}{
            Step $1$: 
            $v_{k + 1}^{\ast} \leftarrow \hat{P} v_k$ 
            \tcp*[r]{$g_{k+1}^{\ast}(\mu) \assign \mu g_{k}(\mu)$}
            
            Step $2$: 
            $ v_{k + 1}^{\bot} \leftarrow v_{k + 1}^{\ast} -
              \sum_{i=0}^{k}\langle v_{k + 1}^{\ast}, v_i \rangle v_i 
              $ \tcp*[r]{$g_{k+1}^{\bot}(\mu) \assign g_{k+1}^{\ast}(\mu)
            - \sum_{i=0}^{k}\langle v_{k + 1}^{\ast}, v_i \rangle  g_{i}(\mu)
            $}

            Step $3$: 
            $v_{k + 1} \leftarrow v_{k + 1}^{\bot} / \| v_{k +
              1}^{\bot} \|$ \tcp*{$g_{k+1}(\mu) \assign g_{k+1}^{\bot}(\mu) / 
              \| v_{k + 1}^{\bot}\|
              $}
              
            $z \leftarrow z + \alpha_{k + 1, l} v_{k + 1}$
        }
        $Z_{:, l} \leftarrow z$ 
    }
    \KwRet{Z} 
    \label{alg:OptBasisFilteringRaw}
\end{algorithm}

When focusing solely on Condition 1, one can readily think of the fundamental Gram-Schmidt process,   
which generates a sequence of orthonormal vectors through a series of iterative steps: each subsequent basis vector is derived by 1) orthogonalization with respect to \textit{all} the previously obtained vectors, and 2) normalization. 

Moreover, with a slight generalization, Condition 2 can also be met. 
As illustrated in our \textsc{OptBasisFiltering} algorithm (Algorithm~\ref{alg:OptBasisFilteringRaw}), 
besides Steps 2-3 taken directly from the Gram-Schmidt process to ensure orthonormality, 
there is an additional Step 1 that guarantees the existence of the \textit{subsequent accompanying polynomial}.
To show this, we can write out the accompanying polynomial in each step. 
Inductively, assuming that the accompanying polynomials 
for the formerly obtained basis vectors are $g_0,\cdots, g_k$, 
we can observe immediately from the algorithmic flow that the $(k+1)$-th 
accompanying polynomial is
\begin{align}
  g_{k+1}(\mu) := 
  \big( \mu g_k(\mu) - \sum_{i=0}^{k} \langle v_k^{\ast}, v_i \rangle  g_i(\mu) \big) / {\|v_{k+1}^{\bot}\|}, 
\label{eq:implicit-poly-raw}
\end{align}
with $g_0(\mu) = 1 / \|x\|$ as the initial step.
Since for each $(k_1, k_2)$, 
$x^\mathrm{T} g_{k_2} (\hat{P}) g_{k_1} (\hat{P}) x = v_{k_1}^\mathrm{T} v_{k_2} = \delta_{k_1 k_2}$ establishes, 
the sequence $g_0,\cdots, g_K$
is exactly the optimal basis in Definition~\ref{def:opt_basis}.
Thus, by solving the vectors in the optimal vector basis in order   
and at the same time apply them in filtering by Equation~\eqref{eq:vecbasis}, 
\textbf{we can make implicit yet exact use of the optimal polynomial basis}. 
Thus, 
we can make implicit and exact use of the optimal polynomial basis 
via solving the optimal vector basis and applying them by Equation~\eqref{eq:vecbasis}. 
The cost, due to the recursive conducting over Step 2 until $v_K$ is obtained, 
is in total $O (K|E|+K^2|V|)$. 

\begin{remark}
  It is revealed by Equation~\eqref{eq:implicit-poly-raw} that 
  we have in fact provided an \textit{alternative solution} to the optimal basis.   
  However, notice that we never need to explicitly compute the polynomial series.
\end{remark}

\textbf{Achieving $O(K|E|\hspace{-1mm}+\hspace{-1mm}K|V|)$ Time Complexity.\quad} 
% \textbf{Efficiency.} 
We can further reduce the cost to $O (K|E|+K|V|)$ 
by Proposition~\ref{prop:onlytwo}, 
which shows that in Step 2, instead of subtracting all the former vectors, 
we just need to subtract $v_k$ and $v_{k - 1}$ from $v_{k + 1}^{\ast}$.

\begin{tcolorbox}[boxrule=0.pt,height=12mm,valign=center,colback=blue!3!white]
  \begin{proposition}
      In Algorithm \ref{alg:OptBasisFilteringRaw}, $v^{\ast}_{k + 1}$ is only denpendent with $v_k$ and $v_{k - 1}$.
    \label{prop:onlytwo}
    \end{proposition}
\end{tcolorbox}

\begin{proof}
Please check Appendix~\ref{sec:proof-of-vec3term}.
\end{proof}

\begin{remark}
The proof is hugely inspired by the core proof part of the Theorem \ref{thm:3term}
(Appendix \ref{sec:proof-of-3term}, the three-term recurrences theorem for orthogonal
polynomials), which shows that 
{ $x p_k (x) $is only relevant to $p_{k + 1}
(x)$, $p_k (x)$ and $p_{k - 1} (x)$}. 
The difference is just a shift of consideration of 
the inner-product space from polynomials to vectors.
\end{remark}

% \begin{algorithm}[htp]
\begin{algorithm}[t]
    \SetNoFillComment
    \SetAlgoNoLine
    \DontPrintSemicolon
    \KwIn{
        Normalized graph adjacency $\hat{P}$; 
        \textbf{Two} solved basis vectors $v_{k - 1}, v_k$ ($k \geq 0$)
    }
    % \tcp*{Implicit input: $g_{k-1}, g_{k} (k \leq 0)$. }
    \KwOut{$v_{k + 1}$} 
    \BlankLine
    \Indp
    % \Foreach{channel \in $X$}{}
    Step $1$: $v_{k + 1}^{\ast} \leftarrow \hat{P} v_k$ 
    % \tcp*{$g_{k+1}^{\ast}(\mu) \assign \mu g_{k}(\mu)$}
    
    Step $2$: $v_{k + 1}^{\bot} \leftarrow v_{k + 1}^{\ast} -
      \langle v_{k + 1}^{\ast}, v_k \rangle v_k - \langle v_{k + 1}^{\ast}, v_{k - 1}\rangle v_{k-1}$ 
      % \tcp*{$g_{k+1}^{\bot}(\mu) \assign g_{k+1}^{\ast}(\mu) - \langle v_{k + 1}^{\ast}, v_k \rangle  g_{k}(\mu) - \langle v_{k + 1}^{\ast}, v_{k-1} \rangle  g_{k-1}(\mu)$}

    Step $3$: $v_{k + 1} \leftarrow v_{k + 1}^{\bot} / \| v_{k +
      1}^{\bot} \|$ 
      % \tcp*{$g_{k+1}(\mu) \assign g_{k+1}^{\bot}(\mu) / \| v_{k + 1}^{\bot}\|$ }
    
    \KwRet{$v_{k + 1}$}
    \caption{
        \textsc{ObtainNextBasisVector} 
        % {
        %   % \setstretch{1.0}
        %   \footnotesize 
        %   \linespread{0.8} 
        %   (In the comment, we write the 
        %   the $(k+1)$-th optimal basis polynomial \textcolor{blue}{$g_{k+1}(\cdot)$}  
        %   based on \textcolor{blue}{$g_{k}(\cdot)$} and 
        %   \textcolor{blue}{$g_{k-1}(\cdot)$}   
        %   that is implicitly used but never solved explicitly.
        %   For initialization, $g_{-1}=0, g_{0}=1/\|x\|$.
        %   )
        % }
    }
    \label{alg:nextbasis}
\end{algorithm}

By Proposition~\ref{prop:onlytwo}, 
we substitute Steps 1-3 in Algorithm~\ref{alg:OptBasisFilteringRaw}
% , namely the \textsc{ObtainNextBasisVector} process,  
by Algorithm~\ref{alg:nextbasis}.
Note that 
we define $v_{-1}:=\vec{0}, g_{-1}(\mu):=0$
for consistency and simplicity of presentation.
% before the iteration starts. 
% The final version of \textsc{OptBasisFiltering} algorithm can be found  in Appendix~\ref{sec:pseudo_optbasis}.
% \textbf{Formulation of OptBasisGNN.\quad}
% \revised{
% For clarity, 
% we put the illustrated final version of \textsc{OptBasisFiltering} algorithm in Appendix~\ref{sec:pseudo_optbasis},  
% which serves as the core part of the complete OptBasisGNN.
% }
% \footnote{The other parts of OptBasisGNN are MLP layers and softmax layers, same as in \textsc{FavardGNN} (Algorithm \ref{alg:favardgnn_cls}).}
The improved \textsc{OptBasisFiltering} algorithm 
serves as the core part of the complete OptBasisGNN.
The processes on all channels are conducted in parallel.
Please check the Pytorch-style pseudo-code in Appendix~\ref{sec:pseudo_torch_OptBasis}.

\subsection{More on the Implicitly Solved Polynomial Basis}

This section is a more in-depth discussion about the nature of our method, that is, we implicitly determine the optimal polynomials by \textit{three-term recurrence relations} rather than the \textit{weight function}.  

We begin with a lemma. The proof can be found in Appendix~\ref{sec:proof-of-consistent-equation}.

\begin{lemma}
  In Algorithm~\ref{alg:nextbasis}, 
  $\|v^{\bot}_{k}\|  = \langle v_{k + 1}^{\ast}, v_{k-1} \rangle$
  \label{lemma:consistent}.
\end{lemma}

This lemma soonly leads to the following theorem.\footnote{Here,  $x$ is the input signal.}

\begin{tcolorbox}[boxrule=0.pt,height=52mm,valign=top,colback=blue!3!white]
\begin{theorem}
[Three-term Recurrences of Accompanying Polynomials (Informal)]The process for deriving the vector basis  
correspondingly defines the optimal polynomial basis 
through the following \textit{three-term relation}: 
\begin{align*}
    & \mymathinlinehl{\| v^{\bot}_{k+1}\| }
     % \|v^{\bot}_{k+1}\| 
      g_{k + 1} (\mu) 
      = ( \mu - \langle v_{k + 1}^{\ast},  v_k \rangle) g_k (\mu) 
     \\
    &\quad\quad\quad\quad\quad\quad\quad\quad\quad
    - \mymathinlinehl{\|v^{\bot}_{k}\| } g_{k - 1} (\mu), 
    \\
    &\quad g_{-1}(\mu) \assign 0, \ g_0 (\mu)  = 1 / \|x\|,
    \\ 
     &\quad k = 0,\cdots,K\minus 1.
\end{align*}
\label{thm:3term-implicit}
\end{theorem}
\end{tcolorbox}
% \vspace{-5mm}

\begin{proof}
Combining Proposition~\ref{prop:onlytwo}, 
the accompanyingly derived basis polynomial in Equation~\eqref{eq:implicit-poly-raw} comes to 
\begin{align*}
  {\|v^{\bot}_{k+1}\|} g_{k+1}(\mu) = \mu  g_{k}(\mu)
  -\sum_{i=k-1}^{k} \langle v_{k + 1}^{\ast}, v_i \rangle  g_{k}(\mu). 
\end{align*}
By employing Lemma~\ref{lemma:consistent} on the right-hand side 
and staking the steps, 
the proof is completed.
\end{proof}

This implicit recurring relation revealed in Theorem~\ref{thm:3term-implicit} perfectly matches the three-term formula in Equation~\eqref{eq:formula_orthonormal} 
if we substitute $\|v_k^{\bot}\|$ by $\sqrt{\beta_k}$, 
and $\langle v_{k+1}^{\ast}, v_k \rangle$ by $\gamma_k$. 
This is guaranteed by the orthonormality of the optimal basis 
% with respect to the weight function $f$ 
(Proposition~\ref{prop:exact_weight}) 
and the three-term recurrence formula that constrains any orthonormal polynomial series (Theorem~\ref{thm:far-orthonormal}). 
% (Theorem~\ref{thm:far-orthonormal}) ensures that any orthonormal polynomial series conforms to formula~\eqref{eq:formula_implicit_poly}.
From this perspective,
OptBasisGNN is a \textbf{particular case} of FavardGNN. 
FavardGNN is possible to reach the whole space of orthonormal bases, 
among which OptBasisGNN employs the ones that promise optimal convergence property.

Let us recall the Favard's theorem (Theorem~\ref{thm:3term_orthonormal}) and three-term recurrence theorem (Theorem~\ref{thm:far-orthonormal})
from a different perspective: 
An orthonormal polynomial series can be defined either through a \textit{weight function} or a \textit{recurrence relation} of a specific formula. 
We adopt the latter definition,    
bypassing the need for eigendecomposition as a prerequisite for the weight function.
Moreover, our adoption of such a way of definition is hidden behind the calculation of vector basis.

% \textbf{Relation to FavardGNN.\quad} 
% OptBasisGNN is a \textbf{particular case} of FavardGNN. 
% FavardGNN is possible to reach the whole space of orthonormal bases, 
% among all these bases, 
% while OptBasis is the one that promises optimal convergence property.

\begin{algorithm}[htp]
    \SetAlgoNoLine
    \KwIn{
        Input signals $X$ with $d$ channels; 
        Normalized graph adjacency $\hat{P}$;
        Order $K$
        % Coefficients $\alpha$
            }
    \Parameter{$\alpha$}
    \KwOut{Filtered signals $Z$} 
    \BlankLine
    % \Indp
    % \Foreach{channel \in $X$}{}
    $v_{-1} \leftarrow 0$\\
    \For{$l=0$ \KwTo $d-1$}{ 
        $x \leftarrow X_{:, l}$
        ,\  $v_0 \leftarrow x / \| x \| $
        ,\  $z \leftarrow \alpha_{0, l} v_0$
        \\ \For {$k=0$ \KwTo $K$}{
            $v_{k + 1}
            \leftarrow${\textsc{ObtainNextBasisVector}}($\hat{P}$,$v_k$, $v_{k - 1})$
        \\ $z \leftarrow z + \alpha_{k + 1, l} v_{k + 1}$
        }
        $Z_{:, l} \leftarrow z$
    }
    \KwRet{Z}
    \caption{\textsc{OptBasisFiltering}}
    \label{alg:OptBasisFiltering}
\end{algorithm}
\subsection{Scale Up OptBasisGNN}
\label{sec:scale-up}
\revised{
% A slight generalization can scale OptBasisGNN up to much larger graphs, such as ogbn-papers100M~\cite{ogb}.
% Following previous works that scale up GNNs by decoupling feature propagation and transformation~\cite{Chen2020GBP,Felix2019SGC,he2022chebii},
By slightly generalizing OptBasisGNN, it becomes feasible to scale it up for significantly larger graphs, such as ogbn-papers100M\cite{ogb}. 
This follows the approach of previous works that achieve scalability in GNNs by decoupling feature propagation from transformation~\cite{Chen2020GBP,Felix2019SGC,he2022chebii}.
}
% We modify OptBasisGNN by 
% 1) drop the MLP layer before \textsc{OptBasisFiltering}, \revised{thus, the optimal basis vectors for all channels require just one pass of calculation}; 
% and 2) \textbf{preprocess} the whole set of basis vectors (denote as $V \in R^{d\times(K+1)\times N}$) on CPU, 
% and 3) conduct \textbf{batch training}: for each batch of nodes $\mathcal{B}$,  move the corresponding segment of basis vectors $V[:,:,\mathcal{B}]$ to GPU. 
We make several modifications to OptBasisGNN. First, we remove the MLP layer before \textsc{OptBasisFiltering}, resulting in the optimal basis vectors for all channels being computed in just one pass. 
Second, we preprocess the entire set of basis vectors ($V \in \mathbb{R}^{d\times(K+1)\times N}$) on CPU. 
Third, we adopt batch training, where for each batch of nodes $\mathcal{B}$, the corresponding segment of basis vectors $V[:,:,\mathcal{B}]$ is transferred to the GPU.

\section{Experiments}

\begin{table*}[htb]
  \tiny
  \caption{
  % \small  
  \textbf{Experimental results.}     
    {\tmem{Accuracies $\pm$ $95\%$
    confidence intervals}} are displayed for each model on each dataset. 
    The best-performing two results are highlighted. 
    The results
    of GPRGNN are taken from \citet{He2021bern}. The results of BernNet, ChebNetII and
    JacobiConv are taken from original papers. The results of FavardGNN and OptBasisGNN are the  average of repeating experiments over 20 cross-validation splits.
  }
  \centering
    \begin{tabular}{lllllll}
      \toprule
      Dataset & Chameleon & Squirrel & Actor  & Citeseer & Pubmed\\ 
      $\| V \|$ & 2,277 & 5,201 & 7,600 & 3,327 & 19,717\\
      $\mathcal{H} (G)$ & .23 & .22 & .22  & .74 & .80\\
      \midrule
      MLP & $46.59 \pm 1.84$ & $31.01 \pm 1.18$ & $40.18 \pm 0.55$ &
       $76.52 \pm 0.8$9 & $86.14 \pm 0.25$\\
      % \hline
      GCN~\cite{kipf2016semi} & $60.81 \pm 2.95$ & $45.87 \pm 0.8$ & $33.26 \pm 1.15$ &  
       $79.85 \pm 0.78$ & $86.79 \pm 0.31$\\
      % \hline
      ChebNet~\cite{Defferrard2016cheb} & $59.51 \pm 1.25$ & $40.81 \pm 0.42$ & $37.42 \pm 0.58$ &
       $79.33 \pm 0.57$ & $87.82 \pm 0.24$\\
      % \hline
      ARMA~\cite{arma2021bianchi} & $60.21 \pm 1.00$ & $36.27 \pm 0.62$ & $37.67 \pm 0.54$ & 
      $80.04 \pm 0.55$ & $86.93 \pm 0.24$\\
      % \hline
      APPNP~\cite{klicpera2019appnp} & $52.15 \pm 1.79$ & $35.71 \pm 0.78$ & $39.76 \pm 0.49$ & 
       $80.47 \pm 0.73$ & $88.13 \pm 0.33$
      \\  
      GPR-GNN~\cite{chien2021gprgnn} & $67.49 \pm 1.38 $ & $50.43 \pm 1.89$ & $39.91 \pm 0.62$ & 
       $80.13 \pm 0.8$4 & $88.46 \pm 0.31$
      \\
      BernNet~\cite{He2021bern} & $68.53 \pm 1.68$ & $51.39 \pm 0.92$ & $41.71 \pm 1.12$ & 
       $80.08 \pm 0.75$ & $88.51 \pm 0.39$
      \\
      % \hline
      ChebNetII~\cite{he2022chebii} & $71.37 \pm 1.01$ & $57.72 \pm 0.59$ &  ${41.75 \pm
      1.07}$ &  $80.53 \pm 0.79$ & $88.93 \pm
      0.29$
      \\
      % \hline
      JacobiConv~\cite{Wang2022jacobi} & \bestcell {${74.20 \pm 1.03}$} & $57.38 \pm 1.25$ & $41.17 \pm 0.64$ &  \bestcell $80.78 \pm 0.79$ & $89.62 \pm 0.41$
      \\
      \midrule
      % FavardGNN & $69.82$ & \bestcell ${{61.13}}$ & $\bestcell {{43.34}}$ & $88.13$ & \bestcell ${{81.77}}$ & \bestcell ${{91.03}}$\\
    FavardGNN & $72.32 \pm 1.90$ & \bestcell $63.49 \pm 1.47$ & $\bestcell {43.05 \pm 0.53}$ &  \bestcell $81.89 \pm 0.63$ & \bestcell $90.90 \pm 0.27$\\
      % \hline
      % OptBasisGNN & \bestcell ${72.66}$ & \bestcell ${{62.49}}$ &
      % {\bestcell ${42.87}$} & $88.43$ & {\bestcell {$81.25{\uparrow}$}} & {\bestcell ${90.36}$}\\
OptBasisGNN & \bestcell $74.26 \pm 0.74$ & \bestcell $63.62 \pm 0.76$ &
      {\bestcell $42.39 \pm 0.52$} &  {$80.58 \pm 0.82$} & {\bestcell $90.30 \pm 0.19$}\\
      \bottomrule
    \end{tabular}
    \label{tbl:node_cls}
  \end{table*}
\begin{table*}[htb]
    \centering
    \tiny
    \caption{{\textbf{Experimental results} of large-scale datasets (non-homophilous).} {\tmem{Accuracies $\pm$ standard errors}} are displayed for each model on each dataset. 
    The best-performing two results are highlighted. Results of BernNet and ChebNet are taken from \citet{he2022chebii}. Other results are from \citet{Lim2021large}. \textbf{Note that} for the large Pokec and Wiki datasets, we use the \textit{scaled-up} version of OptBasisGNN, which is introduced in Section ~\ref{sec:scale-up}. 
    } 
    % \begin{tabular}{|l|lllll|}
    \begin{tabular}{llllll}
      \toprule
      Dataset & Penn94 & Genius & Twitch-Gamers & Pokec & Wiki\\
      $\| V \|$ & 41,554 & 421,961 & 168,114 & 1,632,803 & {1,925,342}\\
      $\| E \|$ & 1,362,229 & 984,979 & 6,797,557 & 30,622,564 & {303,434,860}\\
      $\mathcal{H} (G)$ & .470 & .618 & .545 & .445 & .389\\
      \midrule
      MLP & $73.61 \pm 0.40$ & $86.68 \pm 0.09$ & $60.92 \pm 0.07$ & 
      $62.37 \pm 0.02$ & $37.38\pm0.21$ \\
    %   \hline
      GCN~\cite{kipf2016semi} & $82.47 \pm 0.27$ & $87.42 \pm 0.31$ & $62.18 \pm 0.26$ & 
      $75.45 \pm 0.17$ & OOM \\
      % \midrule
      GCNII~\cite{chen2020gcnii} & $82.92 \pm 0.59$ & $90.24 \pm 0.09$ & $63.39 \pm 0.61$ &
       $78.94 \pm 0.11$ & OOM \\
      MixHop~\cite{abu2019mixhop} & $83.47 \pm 0.71$ & $90.58 \pm 0.16$ & $65.64 \pm 0.27$ &
       $81.07 \pm 0.16$ & $ 49.15 \pm 0.26$ \\
      % \midrule
      LINK~\cite{Lim2021large} & $80.79 \pm 0.49$ & $73.56 \pm 0.14$ & $64.85 \pm 0.21$ & 
      $80.54 \pm 0.03$  & $57.11 \pm 0.26$  \\
    %   \hline
      LINKX~\cite{Lim2021large} & $84.71 \pm 0.52$ & $90.77 \pm 0.27$ & \bestcell ${{66.06 \pm 0.19}}$
      &  $82.04 \pm 0.07$  & $59.80 \pm 0.41$ \\
      % \midrule
      GPR-GNN~\cite{chien2021gprgnn} & $83.54 \pm 0.32$ & $90.15 \pm 0.30$ & $62.59 \pm 0.38$ &
      $80.74 \pm 0.22$  & $58.73 \pm 0.34$ \\
      BernNet~\cite{He2021bern} & $83.26 \pm 0.29$ & $90.47 \pm 0.33$ & $64.27 \pm 0.31$ &
      $81.67 \pm 0.17$  & $59.02 \pm 0.29$ \\
    %   \hline
      ChebNetII~\cite{he2022chebii} & \bestcell {${84.86 \pm 0.33}$} & \bestcell {{$90.85 \pm
      0.32$}} & $65.03 \pm 0.27$ &  \bestcell {${82.33 \pm 0.28}$} & \bestcell {$60.95 \pm 0.39$} \\
      % \hline
      \midrule
      FavardGNN & \bestcell ${84.92±  \pm 0.41}$ &  $90.29 \pm 0.14$ & $64.26 \pm  0.12$ & - & - 
      \\
      OptBasisGNN &  ${84.85 \pm 0.39}$ & \bestcell ${{90.83 \pm 0.11}}$ &
      \bestcell {${65.17\pm 0.16}$} & \bestcell ${82.83 \pm 0.04}$ &  \bestcell ${61.85 \pm 0.03}$ 
      \\
      \bottomrule
    \end{tabular}
    \label{tbl:nonHomo}
  \end{table*}

In this section, we conduct a series of comprehensive experiments to demonstrate the effectiveness of the proposed methods.
Experiments consist of node classification tasks on small and large graphs, the learning of multi-channel filters, and a comparison of FavardGNN and OptBasisGNN.

% overall
% In this section, 
% we carry out a variety of comprehensive experiments.
% We run \textbf{node classification} tasks
% to compare our models (FavardGNN and OptBasisGNN) with leading GNN models
% on a variety of open datasets. 
% % Especially, we include large graphs including  \cite{Lim2021large,ogb} 
% % to show the performance of \textbf{scaled-up} OptBasisGNN.
% Furthermore, to demonstrate the performance of our \textbf{scaled-up OptBasisGNN}, we include large graphs \cite{Lim2021large,ogb}, especially papers100M.
% Then, we confirm the superior 
% \textit{convergence property}  
% of our optimal basis over other polynomial bases by 
% an experiment to \textbf{learn multi-channel filters}.
% Lastly, we \textbf{compare our two models}.

% FavardGNN using 
% learned orthonormal basis
% v.s.
% OptBasisGNN using the optimal basis.
% we show that: the only is better than any.
% \textcolor{blue}{[Yuhe: Is this sentence ok?]} 

% we run \textit{node classification} experiments 
% to compare the performance of our proposed models 
% (i.e. FavardGNN and OptBasisGNN) 
% against leading GNNs, 
% using a diverse selection of public graph datasets. 
% Among these datasets, we include large graphs
% and scale up OptBasisGNN to test on them.
% We also show superior \textit{convergence property} of 
% our optimal basis over other polynomial bases by 
% an experiment to \subsection{learn filters}.
% At last, we compare $FavardGNN$ and $OptBasisGNN$.

\subsection{Node Classification}
\paragraph*{Experimental Setup.} 
% 1. datasets and splits
We include medium-sized graph datasets conventionally used in preceding 
graph filtering works, 
including  
three heterophilic datasets (Chameleon, Squirrel, Actor) 
provided by \citet{Pei2020GeomGCN}
and
two citation datasets (PubMed, Citeseer)
provided by \citet{yang2016revisiting} and \citet{sen2008collective}
. 
% split
For all these graphs, we take a $60\%/20\%/20\%$ train/validation/test 
split proportion following former works, e.g. \citet{chien2021gprgnn}. 
We report our results
%  run these datasets using our models 
of twenty runs over random splits with random initialization seeds. 
% 2. baselines & reported results (from bernnet and original paper)
For baselines, we choose sota spectral GNNs.  
% Note that spectral GNNs are very competitive, especially on heterophilic datasets. 
For other experimental settings, please refer to Appendix~\ref{expappendix:nodecls}.
Besides, for evaluation of OptBasisGNN, 
please also check the results in the scalability 
experimental section (Section \ref{sec:exp_scaleup}). 
% 3. learning process and hyperparameter tuning (put into appendix)
% \paragraph*{Results.} 

\textbf{Results.\quad} 
% As shown in Table \ref{tbl:node_cls}, 
% FavardGNN and OptBasisGNN outperform all other baselines 
% on all six datasets 
% except for Chameleon and Cora, 
% and on Chameleon, OptBasisGNN ranks only after JacobiConv.
% Note that the performance gains on Squirrel, Actor and Pubmed are 
% large.  
As shown in Table \ref{tbl:node_cls}, 
FavardGNN and OptBasisGNN outperform most strong baselines. Especially, in Chameleon, Squirrel and Actor, we see a big lift.
The vast selection range and learnable nature of FavardGNN and the optimality of convergence provided by OptBasisGNN both enhance the performance of polynomial filters, and their performances hold flat. 

% \subsection{Scalability of OptBasisGNN.}
\subsection{Node Classification on Large Datasets}
\label{sec:exp_scaleup}
\paragraph*{Experimental Setup.}
% datasets
We perform node classification tasks on 
two large citation networks: ogbn-arxiv and ogbn-papers100M \cite{ogb}, 
and five large non-homophilic networks from the LINKX datasets \cite{Lim2021large}
.
% 3. models
Except for Penn94, Genius and Twitch-Gamers, 
all other mentioned datasets use the scaled version of OptBasisGNN.

For ogbn datasets, 
% the train/validation/test split is given, 
% and 
we run repeating experiments on the given split with ten random model seeds, 
and choose baselines following the scalability experiments in ChebNetII \cite{he2022chebii}.
{For LINKX datasets}, 
% the split proportion is $50\%/25\%/25\%$. 
% Besides the largest Wiki dataset, 
we use the five given splits 
to align with other reported experiment results
for Penn94, Genius, Twitch-Gamer and Pokec.
For Wiki dataset, 
since the splits are not provided, 
we use five random splits. 
For baselines, we choose spectral GNNs 
as well as top-performing spatial  models reported by \citet{Lim2021large}, including LINK, LINKX, 
GCNII~\cite{chen2020gcnii} 
and MixHop~\cite{abu2019mixhop}.  
% 4. 
For more detailed experimental settings, please refer to Appendix \ref{expappendix:nodecls}.

% \paragraph*{Results.}
\begin{table}[h]
    \centering
    \tiny
    \caption{\textbf{Experimental results} of large-scale datasets (ogbn-citation datasets).
    {\tmem{Accuracies $\pm$
    $95\%$ {standard errors}}} are displayed. Besides OptBasisGNN,
    all the reported results are taken from ChebNetII. The dash line in BernNet
    means failing in preprocessing basis vectors in 24 hrs. Fixed splits of
    train/validation/test sets are used. 10 random model seeds are used for
    repeating experiments.}
    \begin{tabular}{lll}
    \toprule
    Dataset & ogbn-arxiv & ogbn-papers100M \\
    $\| V \|$ & 169,343 & 111,059,956 \\
    $\| E \|$ & 1,166,243 & 1,615,685,872 \\
    $\mathcal{H} (G)$ & 0.66 & - \\
    \midrule
    GCN~\cite{kipf2016semi} & $71.74 \pm 0.29 $ & OOM  \\
    ChebNet~\cite{Defferrard2016cheb} & $71.12 \pm 0.22$ & OOM \\
    ARMA~\cite{arma2021bianchi} & $71.47 \pm 0.25$ & OOM \\
    GPR-GNN~\cite{chien2021gprgnn} & $71.78 \pm 0.18$ & $65.89 \pm 0.35$ \\
    BernNet~\cite{He2021bern} & $71.96 \pm 0.27$ & $-$ \\
    SIGN~\cite{sign2020Frasca} & $71.95 \pm 0.12$ & $65.68 \pm 0.16$ \\
    GBP~\cite{Chen2020GBP} & $71.21 \pm 0.17$ & $65.23 \pm 0.31$ \\
    NDLS*~\cite{zhang2021ndls} & $72.24 \pm 0.21$ & $65.61 \pm 0.29$ \\
    ChebNetII~\cite{he2022chebii} & \bestcell ${{72.32 \pm 0.23}}$ & \bestcell ${67.18
    \normalsize{\pm 0.32}}$ \\
    \midrule
    OptBasisGNN & \bestcell ${72.27 \pm 0.15}$ & \bestcell ${{67.22
    \pm 0.15}}$ \\
    \bottomrule
  \end{tabular}
  % \vspace{-1mm}
  \label{tbl:ogbn}
\end{table}

\textbf{Results.\quad} As shown in Table \ref{tbl:nonHomo} and Table \ref{tbl:ogbn}, 
% OptBasisGNN outperforms baselines on all datasets except for Twitch-Gamer 
% and ogbn-arxiv. 
On Penn94, Genius and Twitch-gamer, 
our two models achieve 
comparable results to those of the state-of-the-art spectral methods. 
On ogbn datasets as well as Pokec and Wiki with tens or hundreds of millions of edges,   
we use the scaled version of OptBasisGNN with batch training. 
We do not conduct FavardGNN on these datasets, 
since the basis vectors of FavardGNN cannot be precomputed.  
Notably, on Wiki dataset, the largest non-homophilous dataset, our method surpasses the second top method by nearly one percent, this demonstrates the effectiveness of our scaled-up version of OptBasisGNN.

\subsection{Learning Multi-Channel Filters from Signals}
% This experiment is an extension of 
% learning filters on single channel as in \cite{He2021bern, Balcilar2021Analyzing}. 

% \input{tables/filter_less.tex}

% \paragraph*{Experimental setup.} 
\textbf{Experimental Setup.\quad}
We extend the experiment of 
learning filters conducted by \citet{He2021bern} and \citet{Balcilar2021Analyzing}. 
The differences are twofold: 
First, we consider the case of \textit{multi-channel} input signals 
and learn filters \textit{channelwisely}.
Second, the \textit{only} learnable parameters are the coefficients 
$\alpha$.
Note that the optimization target of this experiment is identical to how the optimal basis was derived by \citet{Wang2022jacobi}
(See Section \ref{sec:SumWang_short}). 
% for polynomials. 
% \vspace{-10em}
\begin{table}[htb]
    \tiny
    \centering
    \caption{Illustration of our multichannel filter learning experiment.}
    \vskip 0.01in
    \begin{tabular}
      {p{1.6cm}p{1.6cm}p{1.6cm}}
      \toprule
      Original Image 
      & 
      Y: Band Reject
      
      Cb: :Low pass
      
      Cr: High Pass & Y: Low Pass
      
      Cb: Band Reject
      
      Cr: Band Reject\\
      \midrule
      \raisebox{0.0\height}{\includegraphics[width=1.5cm,height=1.5cm]{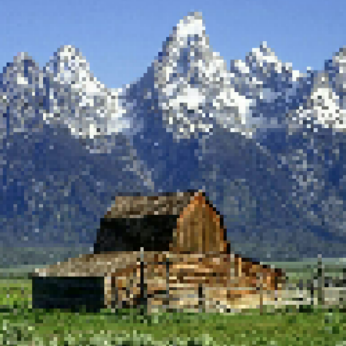}}
      &
      \raisebox{0.0\height}{\includegraphics[width=1.5cm,height=1.5cm]{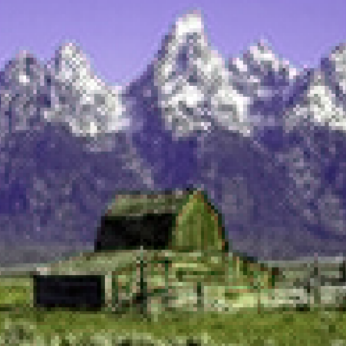}}
      &
      \raisebox{0.0\height}{\includegraphics[width=1.5cm,height=1.5cm]{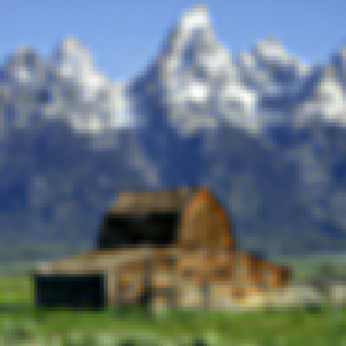}}\\
      \bottomrule
    \end{tabular}
    % \vspace{-4.mm}
    \label{tbl:filtering_less}
  \end{table}
% \vspace{-2em}

We put the practical background of our multichannel experiment in 
YCbCr color space. 
Each $100\times 100$ image is considered as a grid graph
% ($\|V\|=100^2$), 
with input node signals on three channels: Y, Cb and Cr.
Each signal might be 
filtered by complex filtering operations defined in 
\cite{He2021bern}.  
As shown in Table \ref{tbl:filtering_less}, 
using different filters on each channel 
results in different combination effects.
% dataset
We create a synthetic dataset 
% using 15 images, 
% each image is transformed by four \textcolor{blue}{filter combinations}, 
% thus we have 60 samples in total. 
with 60 samples from 15 original images. 
More about the synthetic dataset are in Appendix \ref{expappendix:regression}. 

Following \citet{He2021bern}, 
we use input signals $X$ and the true filtered signals $Y$ 
to supervise the learning process of $\alpha$. 
The optimization goal is to minimize $\frac{1}{2} \|Z-Y\|_2^2$, 
where $Z$ is the output multi-channel signal  defined in Equation~\eqref{eq:MultiChannelFilter}.
During training, 
we use an Adam optimizer with a learning rate of $0.1$ and a weight decay of 
 $5\mathrm{e}{-4}$.
We allow a maximum of $500$ epochs, 
and stop iteration when the difference of losses between two epochs 
is less than $1\mathrm{e}{-4}$.

% The goal of this experiment is to verify the convergence ability 
% of the optimal basis used in our OptBasisGNN. 
For baselines, we choose the Monomial basis, Bernstein basis, 
Chebyshev basis (with Chebyshev interpolation) 
corresponding to GPR-GNN, BernNet and ChebNetII, respectively. 
We also include {arbitrary} orthonormal basis learned by Favard for comparison.
Note that, 
we learn \textit{different filters on each channel
for all baseline basis} for fairness.

\textbf{Results.} We exhibit 
the mean MSE losses with standard errors of the 60 samples achieved by different bases in Table~\ref{tbl:filter_all}. Optbasis, which promises the best convergence property, demonstrates an overwhelming advantage. 
A special note is needed that, 
the Monomial basis has \textit{not finished converging} at the maximum allowed $500$th epoch. In Section~\ref{sec:exp_compare}, we extend the maximum allowed epochs to 10,000, and use the slowly-converging Monomial basis curve as a counterpoint to the non-converging Favard curve. 

\begin{table}[b]
\centering
% \tiny
\caption{
% \textcolor{blue}{TODO}
Experimental results of the multichannel filtering learning task. 
\textit{MSE loss} $\pm$ \textit{standard errors}
of the 60 samples achieved by different bases are exhibited. 
}
\resizebox{\columnwidth}{!}{%
\begin{tabular}{llllll}
\toprule
{BASIS} & OptBasis                                                & ChebII & Bernstein & Favard & Monomial 
\\ \midrule
\begin{tabular}[c]{@{}l@{}}
{MSE}\\ $\pm$ \footnotesize{STDV}
\end{tabular} 
&
% OptBasis
\begin{tabular}[c]{@{}l@{}}
\textbf{0.0058}\\ $\pm$ \footnotesize{\textbf{0.0157}}
\end{tabular} 
&    
% ChebII
\begin{tabular}[c]{@{}l@{}}
 $0.1501$\\ $\pm$ \footnotesize{0.2433}
\end{tabular} 
&    
% Bernstein
\begin{tabular}[c]{@{}l@{}}
$0.4231$\\ $\pm$ \footnotesize{0.4918}
\end{tabular} 
& 
% Favard
\begin{tabular}[c]{@{}l@{}}
$0.3175$\\ $\pm$ \footnotesize{0.2840}
\end{tabular} 
% Monomial
&   
\begin{tabular}[c]{@{}l@{}}
$3.9076$\\ $\pm$ \footnotesize{2.9263}
\end{tabular} 
\\ \bottomrule
\end{tabular}%
}
\label{tbl:filter_all}
\end{table}

Particularly, in Figure \ref{fig:regression}, we visualize 
% the effect of different bases on the  
% converging rate of MSE loss with
the converging process on 
\textbf{one sample}. 
% to learn a combination of multichannel filter 
% when using different bases on one sample. 
% In 500 epochs, the experiment groups 
% OptBasis, ChebII (Chebyshev basis + Chebyshev interpolation) and FavardGNN 
% `converges', while 
% Figure \ref{fig:regression} show several notable information. 
% Most importantly,
Obviously,  
OptBasis show \textbf{best convergence property} 
in terms of both the fastest speed and smallest MSE error. 
Check Appendix~\ref{expappendix:regression} for more samples.
% Secondly, 
% In 500 epochs, the experimental groups 
% of Monomial basis and Bernstein basis did not converge. 
% OptBasis show best convergence property 
% in terms of both fastest speed and smallest MSE error. 
% \textcolor{blue}{Note that the loss and optimization target of this experiment 
% is totally identical to how the optimal basis was derived by \cite{Wang2022jacobi} 
% (See Section \ref{sec:SumWang_short}).}  

\begin{figure}[h]
    \centering
    \tiny
    \includegraphics[width=0.7\linewidth]{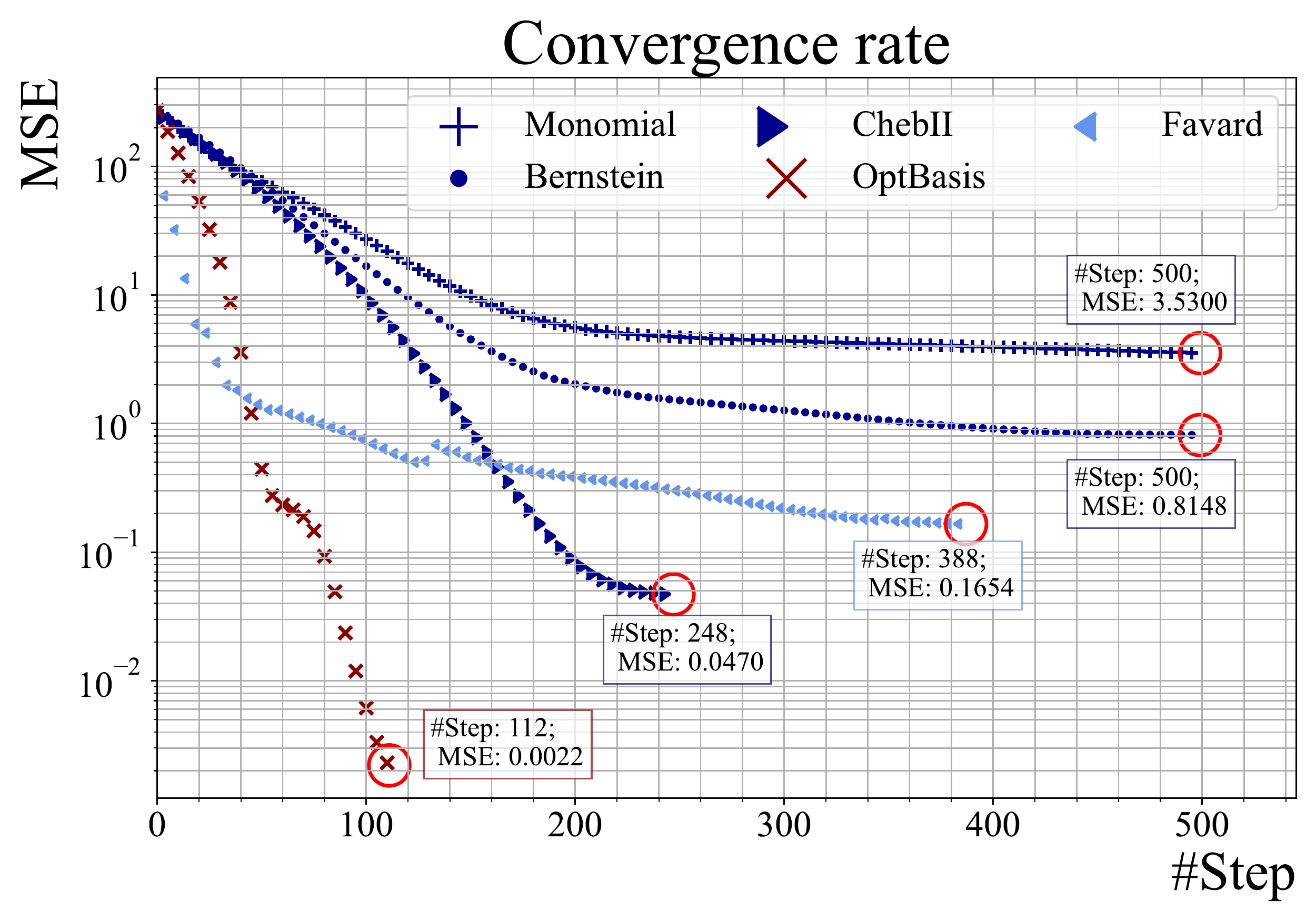}
    % \vspace{-6mm}
    \caption{Convergence rate of minimizing $\frac{1}{2} \|Z-Y\|^2_{2}$ on one sample. 
    \textit{Sample message}: 
    The true filters for this sample are low-pass(Y) / band-reject(Cb) / band-reject(Cr). 
    \textit{Legends}: 
    ChebII means using Chebyshev polynomials combined with 
    interpolation on chebynodes as in ChebNetII \cite{he2022chebii}. Favard means the bases are learned as FavardGNN.  
    In 500 epochs, the experimental groups 
    of the Monomial basis and Bernstein basis did not converge. 
    OptBasis achieves 
    the smallest MSE error in the shortest time.
    }
    \label{fig:regression}
  \end{figure}

% Table \ref{tbl:filtering}.
% Table \ref{tbl:learn_filtering}

\subsection{Non-Convergence of FavardGNN}
\label{sec:exp_compare}
% Based on existing theoretical and empirical analysis, 
% we know that: OptBasisGNN is a particular case of FavardGNN. 
% FavardGNN provides an extensive range of possible basis 
% but lacks explainability, 
% while OptBasisGNN uses a deterministic basis, 
% promising optimal convergence property when considering the 
% a regression target. 
% Both OptBasisGNN and FavardGNN show good performances in node classification tasks.  

% \textbf{Experiment: Non-convergence of FavardGNN.\quad}
Notably, in Figure \ref{fig:regression},  
% though the curve of Favard was considered to have converged at the $248$th epoch, 
an obvious \textit{bump} appeared near the $130$th epoch. We now re-examine the non-convergence problem of FavardGNN (Section \ref{sec:weakness}). 
We rerun the multi-channel filter learning task by canceling early stopping and stretching the epoch number to 10,000. As shown in Figure~\ref{fig:regression_10000} (left), the curve of Favard bump several times. In contrast with Favard is the Monomial basis, though showing an inferior performance in Table~\ref{tbl:filter_all}, it converges slowly but stably. 
We observe a similar phenomenon with a node classification setup in Figure \ref{fig:regression_10000} (right) (See Appendix \ref{expappendix:bump} for details).
% To show this clearer, we let FavardGNN and GPRGNN(which uses the Monomial basis for classification) fit the whole set of nodes, and move \textrm{dropout} and \textrm{Relu} layers. 
Still, very large bumps appear. Such a phenomenon might seem contradictory to the outstanding performance of FavardGNN in node classification tasks. We owe the good performances in Table~\ref{tbl:node_cls} and ~\ref{tbl:nonHomo} to the early stop mechanism. 

\begin{figure}[htp]
% \vspace{-3mm}
  \begin{subfigure}{.5\textwidth}
    \centering
    \includegraphics[width=.48\textwidth]{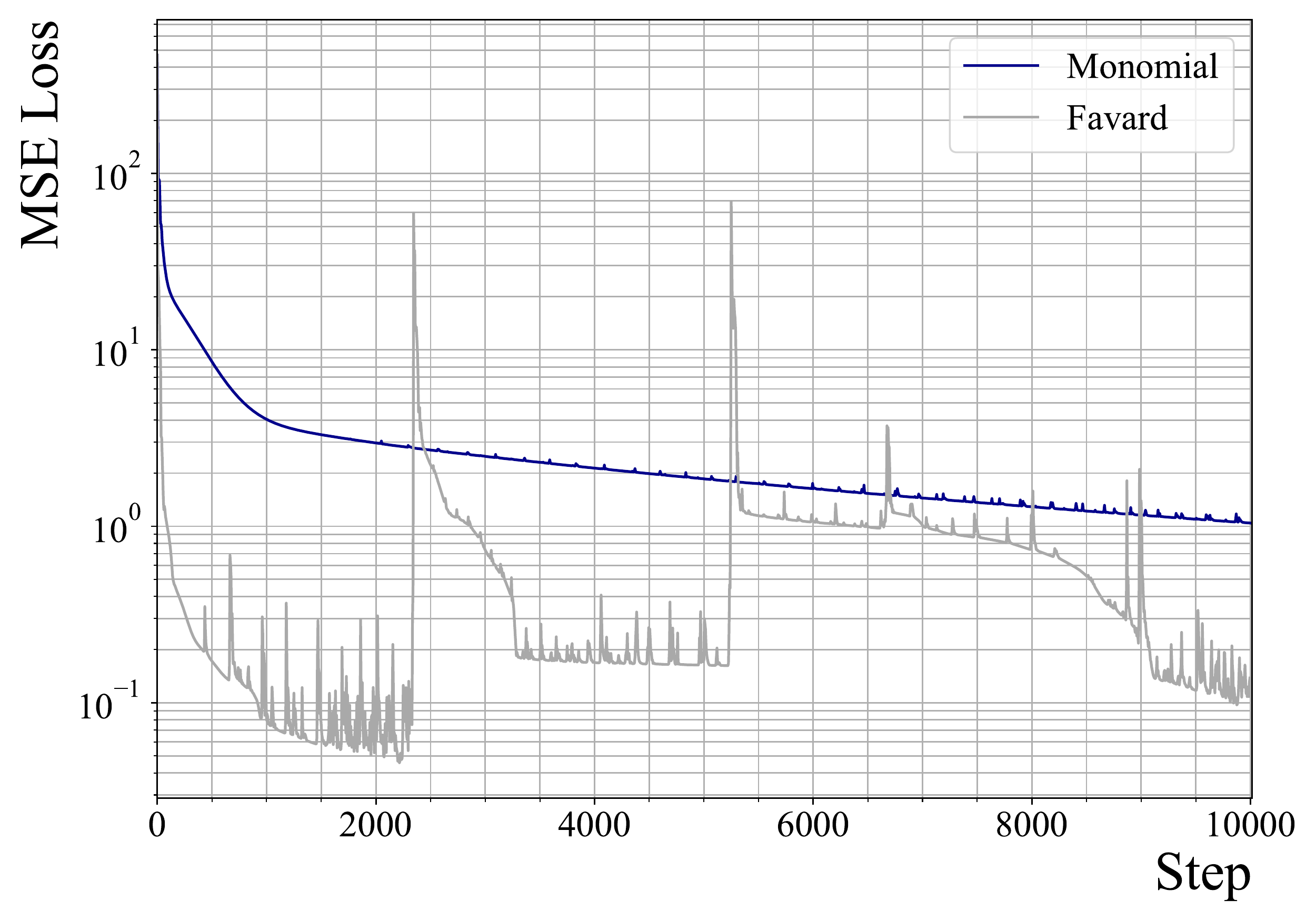}  
    \includegraphics[width=.48\textwidth]{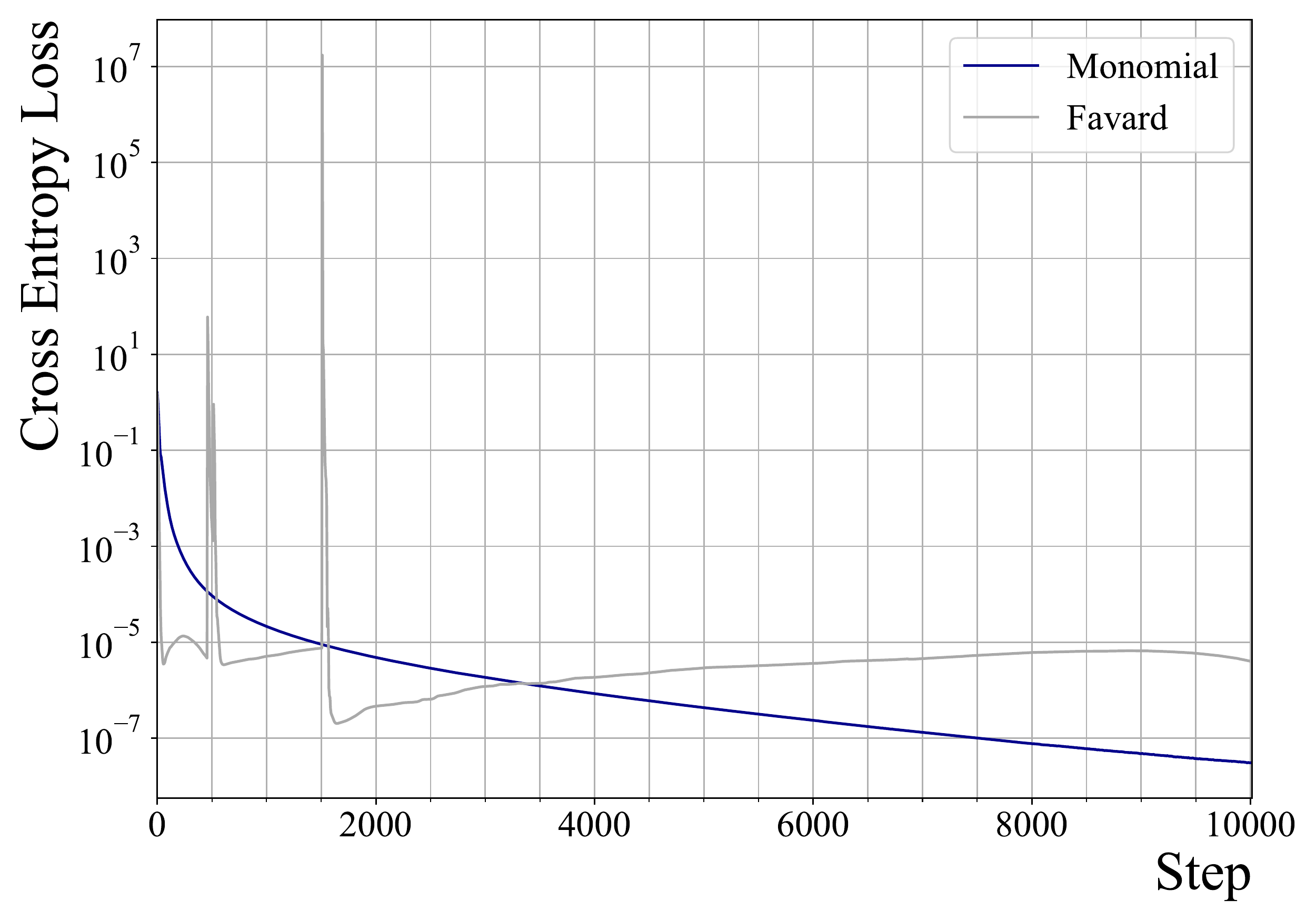}  
  \end{subfigure}
  \caption{Drop of loss in 10,000 epochs. \textit{Left}: MSE loss of regression task on one sample. 
  \textit{Right}: Cross entropy loss of classification problem on the Chameleon dataset. 
  Models based on \underline{Monomial basis} converge slowly, but stably. 
  while \underline{FavardGNNs} don't converge. 
  For the convergence curve for \underline{OptBasis}, please check Figure~\ref{fig:regression}. It converges much faster than Monomial Basis. 
  }
  \label{fig:regression_10000}
  \end{figure}

\section{Conclusion}
In this paper, 
we tackle the fundamental challenges of basis learning and computation in polynomial filters. 
We propose two models: FavardGNN and OptBasisGNN. 
FavardGNN learns arbitrary basis from the whole space of orthonormal polynomials, 
which is rooted in classical theorems in orthonormal polynomials.
OptBasisGNN leverages the optimal basis defined by \citet{Wang2022jacobi} efficiently, 
which was thought unsolvable.   
Extensive experiments are conducted to demonstrate the effectiveness of our proposed models. An interesting future direction is to derive a convex and easier-to-optimize algorithm for FavardGNN.

\section*{Acknowledgements}
This research was supported in part by the major key project of PCL (PCL2021A12), by National Natural Science Foundation of China (No. U2241212, No. 61972401, No. 61932001, No. 61832017), 
by Beijing Natural Science Foundation (No. 4222028), by Beijing Outstanding Young Scientist Program No.BJJWZYJH012019100020098, by Alibaba Group through Alibaba Innovative Research Program, and by Huawei-Renmin University joint program on Information Retrieval. We also wish to acknowledge the support provided by Engineering Research Center of Next-Generation Intelligent Search and Recommendation, Ministry of Education. Additionally, we acknowledge the support from Intelligent Social Governance Interdisciplinary Platform, Major Innovation \& Planning Interdisciplinary Platform for the “Double-First Class” Initiative, Public Policy and Decision-making Research Lab, Public Computing Cloud, Renmin University of China. 

% \bibliography{example_paper}
% \bibliography{references.bib}
\Urlmuskip=0mu plus 1mu \relax
\bibliography{./references.bib}
\bibliographystyle{icml2023}

\newpage
\appendix
\onecolumn
\section{Notations}

% \subsection{Definition of orthogonal polynomials }

% \subsection{Notations}
\label{sec:notations}
\begin{table}[H]
% \small
\caption{Summation of notations in this paper. }
\label{tbl:notations}
\begin{tabularx}{1.0\columnwidth}
  % \label{tbl:notations}
  {lX}
    \toprule
    {\tmstrong{Notation}} & {\tmstrong{Description}}\\
    \hline
    % \multicolumn{2}{c}{Graph Related} \\
    % \hline
    $G = (V, E)$ & Undirected, connected graph\\
    % \hline
    $N$ & Number of nodes in $G$\\
    % \hline
    $\hat{P}$ & Symmetric-normalized adjacency matrix of $G$. \\
    % \hline
    $\hat{L}$ & Normalized Laplacian matrix of $G$. $\hat{L}$ = $I -
    \hat{P}$.\\
    % \hline
    $\lambda_i$ & The $i$-th eigenvalue of $\hat{L}$.\\
    % \hline
    $\mu_i$ & The $i$-th eigenvalue of $\hat{P}$. $\mu_i = 1 - \lambda_i .$\\
    % \hline
    $U$ & Eigen vectors of $\hat{L}$ and $\hat{P}$. \\
    \midrule
    % Graph Signal Related & \\
    % \multicolumn{2}{c}{Graph Signal Related} \\
    % \hline
    $x$ & Input signal on $1$ channel.\\
    % \hline
    $X \in \mathbb{R}^{N \times d}$ & Input features / Input signals on $d$
    channels.\\
    % \hline
    $Z \in \mathbb{R}^{N \times d}$ & Filtered signals. \\
    % \hline
    $h (\cdot),$ $b (\cdot)$ & Filtering function defined on $\hat{L}$ and
    $\hat{P}$, respectively. $h (\lambda) \equiv b (1 - \lambda) .$\\
    % \hline
    $h_i (\cdot)$, $b_i (\cdot)$ & Filtering function on the $i$th signal
    channel. $X_{i, :} = h_i (\hat{L}) Z_{i, :}$.\\
    % \hline
    $h (\hat{L}) x,$ $b ( \hat{P}) x$ & Filtering operation on signal $x$. $h
    (\hat{L}) \equiv b ( \hat{P})$.\\
    \midrule
    % \midrule
    % Polynomial Related & \\
    % \multicolumn{2}{c}{Polynomial Related} \\
    % \hline
    $\{ g_k (\cdot) \}_{k = 0}^K$ & A polynomial basis of truncated order
    $K$.\\
    % \hline
    $\{ \alpha_k \}_{k = 0}^K$ & Coefficients above a basis. i.e. $h (\lambda)
    \approx \sum_{k = 0}^K \alpha_k g_k (\lambda) .$\\
    \bottomrule
  \end{tabularx}
\end{table}

\section{Proofs}
\label{sec:appendix-proofs}
Most subsections here are for the convenience of interested readers. 
We provided our proofs about the theorems (except for the original form of Favard's Theorem) and their relations used across our paper, although the theorems can be found in early chapters of monographs about orthogonal polynomials~\cite{gautschi2004orthogonal, simon2014spectral}. 
We assume a relatively minimal prior background in orthogonal polynomials.   

% % and the proofs can be found in  
% we provide some proofs footed on a minimum background in Appendix \ref{sec:proof-of-3term} to \ref{sec:proof-of-favard-orthonormal} for the convenience to check through.

% \subsection{Proof of Theorem \ref{thm:3term}}
\subsection{Three-term Recurrences for Orthogonal Polynomials (With Proof)}
\label{sec:proof-of-3term}

\begin{theorem}[\textbf{Three-term Recurrences for Orthogonal Polynomials}]\citep[p.~12]{simon54orthogonal}
  For any orthogonal polynomial series $\{ p_k (x) \}_{k = 0}^{\infty}$, 
  suppose that the leading coefficients of all polynomial are positive, 
  the series satisfies the recurrence relation: 
\begin{align*}
   %  \label{eq:formula_orthogonal}
% \[
    & p_{k + 1} (x) = (A_k x + B_k) p_k (x) + C_k p_{k - 1} (x), \notag \\
    & p_{-1}(x) \assign 0,  
    % \notag \\
    A_k, C_k \in \mathbb{R}^{+}, \ B_k \in \mathbb{R}, \  k \geq 0.    
\end{align*} 
\label{thm:3term}
\end{theorem}

% {\textbf{(Three term recurrence for Orthogonal
% Polynomials)}} A sequence of orthogonal polynomials $\{ p_n (x) \}_{n =
% 0}^{\infty}$ satisfies
% \[ p_{n + 1} (x) = (A_n x + B_n) p_n (x) + C_n p_{n - 1} (x), \quad n = 1, 2,
%    3, \ldots . \]

\begin{proof}
  The core part of this proof is that $x p_k$ is orthogonal to $p_i$ for $i
  \leq k - 2$, i.e.
  \[ \langle x p_k, p_i \rangle = 0, \quad i \leq k - 2. \]
  Since $x p_k (x)$ is of order $k + 1$, we can rewrite $x p_k (x)$ into the
  combination of first $k + 1$ polynomials of the basis:
  \begin{align}
   \label{eq:proof_of_3term_expansion}
   x p_k (x) = \alpha_{k, k + 1} p_{k + 1} (x) + \alpha_{k, k} p_k (x) +
   a_{k, k - 1} p_{k - 1} (x) + \cdots + \alpha_{k, 0} p_0 (x)
  \end{align}
   
  or in short,
  \[ x p_k (x) = \sum_{j = k + 1}^0 \alpha_{k, j} p_j (x) . \]
  Project each term onto $p_i (x)$,
  \[ \langle x p_k (x), p_i (x) \rangle = \sum_{j = k + 1}^0 \alpha_{k, j}
     \langle p_j (x), p_i (x) \rangle . \]

  Using the orthogonality among $\{ p_k (x) \}_{k = 0}^{\infty}$, 
  we have

  \begin{align}
   \label{eq:proof_of_three_term_coef}
   \langle x p_k (x), p_i (x) \rangle = \langle \alpha_{k, i} p_i (x), p_i
   (x) \rangle \Rightarrow \alpha_{k, i} = \frac{\langle x p_k (x), p_i (x)
   \rangle}{\langle p_i (x), p_i (x) \rangle} . 
  \end{align} 

  Next, we show $\langle x p_k (x), p_i (x) \rangle = 0$ when $i \leq k - 2$.
  Since $\langle x p_k (x), p_i (x) \rangle \equiv \langle p_k (x), x p_i (x)
  \rangle$, it is equivalent to show $\langle p_k (x), x p_i (x) \rangle = 0$.
  
  When $i \leq k - 2$, applying $x p_i (x) = \sum_{j = 0}^{i + 1} \alpha_{i,
  j} p_j (x)$ and the orthogonality between $p_j (x)$ and $p_k (x)$ when $j
  \neq k$, we get
  \[ \langle p_k, x p_i (x) \rangle = \sum_{j = 0}^{i + 1} \alpha_{i, j}
     \langle p_k, p_j (x) \rangle \xequal{j \leq k - 1} 0 \Rightarrow \langle
     x p_k, p_i (x) \rangle = 0. 
   \]
  Therefore, { $x p_k (x) $is only relevant to $p_{k + 1}
  (x)$, $p_k (x)$ and $p_{k - 1} (x)$}. By shifting items, we soonly get that:
  { $p_{ k + 1} (x) $is only relevant to $x p_k (x)$, $p_k (x)$
  and $p_{k - 1} (x)$}.

  At last, we show that, 
  by regularizing the leading coefficients $A_k$ to be positive, ${C_k}> 0$. 
  Firstly, since the leading coefficients are positive, 
  $\{\alpha_k\}_{k=0}^{\infty}$ defined in Equation~\eqref{eq:proof_of_3term_expansion} 
  are positive.
  Then, notice from Equation~\eqref{eq:proof_of_three_term_coef}, 
  we get 
  \[
   - \frac{C_k}{A_k} = \alpha_{k, k-1} = 
   \frac{\langle x p_k (x), p_{k-1} (x)\rangle}{\langle p_{k-1} (x), p_{k-1} (x) \rangle}
  = \frac{\langle  p_k (x), xp_{k-1} (x)\rangle}{\langle p_{k-1} (x), p_{k-1} (x) \rangle}
  =  \frac{\alpha_{k-1,k}}{\langle p_{k-1} (x), p_{k-1} (x) \rangle}.
 \]
  
  We have finished our proof.
\end{proof}

\subsection{Favard's Theorem (Monomial Case)}
\label{sec:favard-monomial}
\begin{theorem}[\textbf{Favard's Theorem}]\cite{favard1935polynomes}
  If a sequence of monic polynomials $\{ P_n \}_{n = 0}^{\infty}$ satisfies a
  three-term recurrence relation
%   \vspace{-1.5em}
  \[ P_{n + 1} (x) = \left( {x - \gamma_n}  \right) P_n (x) - \beta_n P_{n -
     1} (x), \]
with $\gamma_n, \beta_n \in \mathbb{R}, \beta_n > 0$, 
then $\{ P_n \}_{n = 0}^{\infty}$ is orthogonal with respect to some 
positive weight function.
\label{thm:favard-monic}
\end{theorem}
% \subsection{Proof of Corollary \ref{thm:far} }
\subsection{Favard's Theorem (General Case) (With Proof)}
\label{sec:proof-of-far}
% {\dueto{Favard's Theorem; general case}}
\begin{corollary}[\textbf{Favard's Theorem; general case}]
   If a sequence of polynomials $\{ P_n \}_n^{\infty}$ statisfies a three-term recurrence
   relation
   \[ P_{n + 1} (x) = \left( \varsigma_n {x - \gamma_n}  \right) P_n (x) -
      \beta_n P_{n - 1} (x), \]
   with $\gamma_n, \beta_n, \varsigma_n \in \mathbb{R}, \varsigma_n \neq 0,
   \beta_n / \varsigma_n > 0$, then there exists a positive weight function 
   $w$ such that $\{ P_n \}_{n = 0}^{\infty}$ is orthogonal with
   respect to the inner product $\langle p, q \rangle = \int_{\mathbb{R}} p(x) q(x)
   w(x) \mathd x$.   
\end{corollary}

\begin{proof}
  Set $\gamma_n^{\ast} = \frac{\alpha_n}{\varsigma_n},  \beta_n^{\ast} =
  \frac{\beta_n}{\varsigma_n}$. Then we can construct a sequence of
  polynomials $\{ P_n^{\ast} \}_n^{\infty}$.
  
  \textbf{Case 1}: For $n = 0$ and $n = 1$, set $P_n^{\ast} \assign P_n (x) /
  \hat{P_n} (x)$ .
  
  \textbf{Case 2}:  For $n \geq 2$, define $P_n^{\ast} (x)$ by the three-term recurrences:
  \[ P_{n + 1}^{\ast} (x) \assign \left( {x - \gamma_n^{\ast}}  \right)
     P_n^{\ast} (x) - \beta_n^{\ast} P_{n - 1}^{\ast} (x) . \]
  According to Theorem \ref{thm:favard-monic}, $\{ P_n^{\ast} \}_n $ is an
  orthogonal basis. Since $P_n^{}$ is scaled $P_n^{\ast}$ by some constant, so
  $\{ P_n^{} \}_n $ \ is also orthogonal.
\end{proof}

\subsection{Proof of Theorem \ref{thm:3term_orthonormal}}
\label{sec:proof-of-3term-orthonormal}
We restate the Theorem of three-term recurrences for orthonormal polynomials
(Theorem~\ref{thm:3term_orthonormal}) as below, and give a proof.

\textbf{(Three Term Recurrences for Orthonormal Polynomials)}
    For orthonormal polynomials $\{ p_k \}_{k=0}^{\infty}$ w.r.t. weight function $w$, 
    suppose that the leading coefficients of all polynomial are positive, 
    there exists the three-term recurrence relation:
    \begin{align*}
      & \sqrt{\beta_{k + 1}} p_{k + 1} (x) 
       = (x - \gamma_k) p_k (x) -
               \sqrt{\beta_k} p_{k - 1} (x), 
      \notag \\
      & p_{-1}(x) \assign 0, \ p_0 (x) = 1 / \sqrt{\beta_0}, 
      \  \gamma_k \in \mathbb{R}, \ \sqrt{\beta_k} \in \mathbb{R}^{+}, \  k \geq 0
   \end{align*}
with $\beta_0 = \int w(x) \mathd x$.
    
\begin{proof}
      \textbf{Case 1}: $k = 0$. $p_k (x)$ is a constant. Suppose it to be $t$, then
      \[ \text{} \langle p_0 (x), p_0 (x) \rangle = t^2 \int_a^b \hspace{0.17em}
         \mathd \alpha \Rightarrow t = 1 / \sqrt{\beta_0} . \]
      \textbf{Case 2}: $k \geq 1$. By Theorem~\ref{thm:3term}, since $\{ p_k \}_{k = 0}^K$ is
      orthogonal, there exist three term recurrences as such:
      \[ p_{k + 1} (x) = (A_k x + B_k) p_k (x) + C_k p_{k - 1} (x), \quad k = 1,
         2, 3, \ldots . \]
      By setting $c_k^{\ast} = \dfrac{1}{A_k}$, $a_k^{\ast} = - \dfrac{B_k}{A_k}$,
      $b_k^{\ast} = - \dfrac{{C_k} }{A_k}$, it can be rewritten into
      \begin{align}
         \label{eq:proof_for_3term_ortho_rewrite}
         c_k^{\ast} p_{k + 1} (x) = (x - a_k^{\ast}) p_k (x) - b_k^{\ast} p_{k -
         1} (x), \quad k = 1, 2, 3, \ldots .
      \end{align}
      
      % \[ c_k^{\ast} p_{k + 1} (x) = (x - a_k^{\ast}) p_k (x) - b_k^{\ast} p_{k -
      %    1} (x), \quad k = 1, 2, 3, \ldots . \]
      Apply dot products with $p_{k - 1} (x)$ to Equation~\eqref{eq:proof_for_3term_ortho_rewrite}, 
      we get
      \begin{eqnarray}
         \label{eq:proof_for_3term_ortho_medium_1}
        \langle x p_k (x), p_{k - 1} (x) \rangle & = & \langle b^{\ast}_k p_{k -
        1} (x), p_{k - 1} (x) \rangle \notag \\
        & \Rightarrow & b^{\ast}_k = \langle x p_k (x), p_{k - 1} (x) \rangle \\
        &  & (k = 1, 2, 3, \ldots) \notag .
      \end{eqnarray}
      Similarly, apply dot products with $p_{k + 1} (x)$, we get:
      \begin{eqnarray}
         \label{eq:proof_for_3term_ortho_medium_2}
        \langle c_k^{\ast} p_{k + 1} (x), p_{k + 1} (x) \rangle & = & \langle x
        p_k (x), p_{k + 1} (x) \rangle  \notag  \\
        & \Rightarrow & c_k^{\ast} = \langle x p_k (x), p_{k + 1} (x) \rangle \notag \\
        & \Rightarrow & c_k^{\ast} = \langle x p_{k + 1} (x), p_k (x) \rangle \\
        &  & (k = 1, 2, 3, \ldots) \notag.
      \end{eqnarray}
      Notice that in Equation~\eqref{eq:proof_for_3term_ortho_medium_2}
      \[ \langle x p_k (x), p_{k + 1} (x) \rangle = \langle p_k (x), x p_{k + 1}
         (x) \rangle \overset{\eqref{eq:proof_for_3term_ortho_medium_1}}{=} b_{k + 1}^{\ast} . \]
      We get:
      \[ c_k^{\ast} = b_{k + 1}^{\ast} . \]
      
      So we can write Equation~\eqref{eq:proof_for_3term_ortho_rewrite} into the form below: \
      \[ b_{k + 1}^{\ast} p_{k + 1} (x) = (x - a_k^{\ast}) p_k (x) - b_k^{\ast}
         p_{k - 1} (x), \quad k = 1, 2, 3, \ldots . \]
      At last, we show $b_k^{\ast} > 0$.
      
      Firstly, recall that $b^{\ast}_k = \langle x p_k (x), p_{k - 1} (x) \rangle
      = \langle p_k (x), x p_{k - 1} (x) \rangle$. Since $x p_{k - 1} (x)$, which
      is of order $k$, can be written into the combination of $\{ p_j \}_{j =
      0}^k$ which the leading coefficients to be non-zero, i.e.
      \[ x p_{k - 1} (x) = a_{k, k} p_k (x) + a_{k, k - 1} p_{k - 1} (x) + \cdots
         + a_{k, 0} p_0 (x) \quad (a_{k, k} \neq 0) \]
      Secondly, since $\langle g (x), g (x) \rangle \equiv \langle - g (x), - g
      (x) \rangle$, we can restrict all the leading coefficients to be positive.
      \[ b_n^{\ast} = \langle p_k (x), x p_{k - 1} (x) \rangle = a_{k, k} > 0. \]
      Thus we have proved $b^{\ast}_k > 0$ holds.
      
      Furthermore, we can rewrite $b_k^{\ast}$ into $\sqrt{\beta_k}$. 
      The proof is finished.
    \end{proof}

\subsection{Proof of Theorem \ref{thm:far-orthonormal}}
\label{sec:proof-of-favard-orthonormal}

We restate Favard's Theorem for orthonormal polynomials
(Theorem~\ref{thm:far-orthonormal}) as below, and give a proof
based on the general case~\ref{sec:proof-of-far}. 

% \textbf{(Favard Theorem; Orthonormal case)}
% If a sequence of polynomials $\{ p_k \}_{k = 0}^{\infty}$ statisfies a
% three-term recurrence relation
% \[  \sqrt{\beta_{k + 1}} p_{k + 1} (x) = (x - \gamma_k) p_k (x) -
%    \sqrt{\beta_k} p_{k - 1} (x), \ k = 0, 1, \cdots \]
% and 
% \[
%   p_{- 1} (x) \equiv 0, \ p_0 (x) \equiv 1 / \sqrt{\beta_0}, 
% \]
% with
% $\gamma_k \in \mathbb{R}$ and $\sqrt{\beta_k} \in \mathbb{R}^+$,
% \tmcolor{blue}{ }then there exists a positive weight function $w$
% such that $\{ p_k \}_{k = 0}^{\infty}$ is orthonormal 
% with respect to $w$, 
% and $\beta_0 = \int_a^b w(x) \mathd x $.

\textbf{(Favard Theorem; Orthonormal case)}
  A polynomial series $\{ p_k \}_{k = 0}^{\infty}$ who satisfies the recurrence relation
  \begin{align*}
    \label{eq:formula_orthonormal}
    & \sqrt{\beta_{k + 1}} p_{k + 1} (x) 
     = (x - \gamma_k) p_k (x) -
            \sqrt{\beta_k} p_{k - 1} (x), 
    \notag \\
    & p_{-1}(x) \assign 0, \ p_0 (x) = 1 / \sqrt{\beta_0}, 
     \gamma_k \in \mathbb{R}, \ \sqrt{\beta_k} \in \mathbb{R}^{+}, \  k \geq 0
\end{align*}
  is orthonormal w.r.t. a weight function $w$ that  $\beta_0 = \int w(x) \mathd x$.

\begin{proof}
  First of all, according to Theorem \ref{thm:far-orthonormal}, 
  the series $\{ p_k \}_{k = 0}^{\infty}$ is orthogonal.
  
  Apply dot products with $p_{k - 1} (x)$, we get
  \begin{eqnarray*}
    \langle x p_k (x), p_{k - 1} (x) \rangle & = & \left\langle \sqrt{\beta_k}
    p_{k - 1} (x), p_{k - 1} (x) \right\rangle\\
    \Rightarrow \langle x p_k (x), p_{k - 1} (x) \rangle & = & \sqrt{\beta_k}
    \langle p_{k - 1} (x), p_{k - 1} (x) \rangle\\
    &  & (k = 0, 1, \ldots) .
  \end{eqnarray*}
  
  Similarily, apply dot products with $p_{k + 1} (x)$, we get:
  \begin{eqnarray*}
    \left\langle \sqrt{\beta_{k + 1}} p_{k + 1} (x), p_{k + 1} (x)
    \right\rangle & = & \langle x p_k (x), p_{k + 1} (x) \rangle\\
    \Rightarrow \sqrt{\beta_{k + 1}} \langle p_{k + 1} (x), p_{k + 1} (x)
    \rangle & = & \langle x p_k (x), p_{k + 1} (x) \rangle\\
    &  & (k = 0, 1, \ldots),
  \end{eqnarray*}
  which can be rewritten as:
  \[ \begin{array}{lll}
       \sqrt{\beta_k} \langle p_k (x), p_k (x) \rangle & = & \langle x p_{k -
       1} (x), p_k (x) \rangle\\
       &  & (k = 1, 2, \ldots),
     \end{array} \]
  Notice that
  \[ \langle x p_{k - 1} (x), p_k (x) \rangle = \langle x p_k (x), p_{k - 1}
     (x) \rangle . \]
  We get:
  \begin{eqnarray*}
    \sqrt{\beta_k} \langle p_k (x), p_k (x) \rangle & = & \langle x p_{k - 1}
    (x), p_k (x) \rangle\\
    & = & \sqrt{\beta_k} \langle p_{k - 1} (x), p_{k - 1} (x) \rangle\\
    \Rightarrow \langle p_k (x), p_k (x) \rangle & = & \langle p_{k - 1} (x),
    p_{k - 1} (x) \rangle\\
    &  & (k = 1, 2, \ldots),
  \end{eqnarray*}
  which indicates that the polynomials $\{ p_k \}_{k = 0}^K$ are same in their
  norm.
  
  Since $p_0 (x) \equiv 1 / \sqrt{\beta_0}$ and $\beta_0 = \int w(x) \mathd x$,
  $\left\langle {p_0}  (x), p_0 (x) \right\rangle 
  =
  \dfrac{1}{{\beta_0}}
  \int w(x) \mathd x$
  =1.
Thus the norm of every polynomial in $\{ p_k\}_{k = 0}^{\infty}$  equals $1$. 
  
  Combining that $\{ p_k \}_{k = 0}^{\infty}$ is orthogonal and $\left\langle
  {p_k}  (x), p_k (x) \right\rangle = 1$ for all $k$, we arrive that $\{ p_k
  \}_{k = 0}^{\infty}$ is an orthonormal basis.
\end{proof}

\newpage
\subsection{Proof of Proposition~\ref{prop:onlytwo}}
\label{sec:proof-of-vec3term}
\begin{proof}
  First, from the construction of each $v_{i + 1}$ (Algorithm~\ref{alg:OptBasisFilteringRaw}, $k = i$), 
  we know that $v_{i + 1}$ is composed of $\{ v_j \}_{j}^{j = i}$ and
  $\hat{P} v_i$. Therefore, $\hat{P} v_i$ can be expressed as a weighted sum of
  $\{ v_j \}_{j = 0}^{i + 1}$, denoted as 
  $
  \hat{P} v_i = t_{i + 1} v_{i + 1} + t_i v_i + \cdots + t_0 v_0
  ~\refstepcounter{equation}(\theequation)
  \label{eq:weightsum}
  $.
  % $\hat{P} v_i = t_{i + 1} v_{i + 1} + t_i v_i + \cdots + t_0 v_0. $
  Second, notice that 
  $
  \langle \hat{P} v_k, v_i \rangle = v_k^T  \hat{P} v_i = \langle v_k, \hat{P} v_i \rangle
  ~\refstepcounter{equation}(\theequation)
  \label{eq:sym}
  $.
  % $\langle \hat{P} v_k, v_i \rangle = v_k^T  \hat{P} v_i
  % = \langle v_k, \hat{P} v_i \rangle$. 
  Thus, for Step 2 in Algorithm~\ref{alg:OptBasisFilteringRaw}, 
  for each $i \in \left[0,1,\cdots, k \right]$ 
  we can rephrase $\langle v_{k+1}^{*}, v_i \rangle$ by: 
  \begin{align*}
    \langle v_{k+1}^{*}, v_i \rangle & 
    \overset{\text{def}}{=} \langle \hat{P} v_k, v_i \rangle 
    \overset{\eqref{eq:sym}}{=} \langle v_k, \hat{P} v_i \rangle
    \\
   &  \overset{\eqref{eq:weightsum}}{=}  \left\langle v_k, \sum_{j = 0}^{i
     + 1} t_j v_j \right\rangle 
    \\
  & = \sum_{j = 0}^{i + 1} t_j \langle v_k,
     v_j \rangle, 
  \end{align*}
  % We get
  % \[ 
  %   \langle v_{k+1}^{*}, v_i \rangle = 
  %   \langle \hat{P} v_k, v_i \rangle = 
  % \langle v_k, \hat{P} v_i \rangle = \left\langle v_k, \sum_{j = 0}^{i
  %    + 1} t_j v_j \right\rangle = \sum_{j = 0}^{i + 1} t_j \langle v_k,
  %    v_j \rangle, \]
  which equals $0$ when $i < k - 1$.
\end{proof}
\subsection{Proof of Lemma~\ref{lemma:consistent}}
\label{sec:proof-of-consistent-equation}

% \begin{tcolorbox}[boxrule=0.pt,height=52mm,valign=top,colback=blue!3!white]
\begin{proof}
    First, notice that
    \[
        \langle v_{k + 1}^{\ast}, v_{k-1} \rangle
        = \langle \hat{P} v_{k}, v_{k-1} \rangle
        = \langle  v_{k}, \hat{P} v_{k-1} \rangle .
    \]
    On the other hand, 
    \[\|v_{k+1}^{\bot}\| 
    = \langle v_{k + 1}, v_{k+1}^{\bot} \rangle 
    = \langle v_{k + 1}, v_{k+1}^{\ast} \rangle 
    = \langle v_{k + 1}, \hat{P} v_{k} \rangle  .
    \]
    So, we get 
    \[
        \|v^{\bot}_{k}\|  
        = \langle v_{k}, \hat{P} v_{k-1} \rangle
        = \langle \hat{P}  v_{k}, v_{k-1} \rangle
        = \langle v_{k + 1}^{\ast}, v_{k-1} \rangle    .
    \]
\end{proof}

Thus, we have finished our proof.
\vspace{100mm}

\section{Pseudo-codes}
% \subsection{Final Version of \textsc{OptBasisFiltering}.}
% \label{sec:pseudo_optbasis}
% The final improved version of \textsc{OptBasisFiltering}
% Algorithm~\ref{alg:OptBasisFilteringRaw} by Proposition~\ref{prop:onlytwo}. 
% \input{algorithms/OptBasisGNN.tex}

\subsection{Pseudo-code for FavardGNN.}
\label{sec:pseudo_torch_Favard}
\begin{algorithm}[H]
\caption{FavardGNN.\textit{Pytorch style}.}
\label{alg:cotextde}
\definecolor{codeblue}{rgb}{0.25,0.5,0.5}
\lstset{
  backgroundcolor=\color{white},
  basicstyle=\fontsize{7.2pt}{7.2pt}\ttfamily\selectfont,
  columns=fullflexible,
  breaklines=true,
  captionpos=b,
  commentstyle=\fontsize{7.2pt}{7.2pt}\color{codeblue},
  keywordstyle=\fontsize{7.2pt}{7.2pt},
% ##  frame=tb,
}
\begin{lstlisting}[language=python]
# f: raw feature dimension
# d: hidden dimension, or number of channels
# N: number of nodes
# K: order of polynomial basis
# X(Nxd): Input features 
# P(NxN): Sym-normalized adjacency matrix 
# Coef(dx(K+1)): coefficient matrix
# SqrtBeta(dx(K+1)): Coefficients for three-term recurrences 
# Gamma(dx(K+1)): Coefficients for three-term recurrences 


# Transfer raw input in signals 
X = ReLU(MLP(X.dropout())).dropout()  # (Nxd)

SqrtBeta = torch.clamp(norm, 1e-2)

# Process H_0
H_0 = X / SqrtBeta[:,0]    # (Nxd)

Z = torch.zeros_like(X)
# Add to the final representation
Z = Z + torch.einsum('Nd,d->Nd', H_0, Coef[:,0])  

last_H = H_0
second_last_H = torch.zeros_like(H_0)

for k in range(1, K):
    # Three-term Recurrence Formula for Orthonormal Polynomials
    H_k = P @ last_H   # (Nxd)
    H_k = H_k - Gamma[k,:].unsqueeze(0)*last_H - SqrtBeta[k,:].unsqueeze(0)*second_last_H
    H_k = H_k / SqrtBeta[k+1,:].unsqueeze(0)

    # Add to the final representation
    Z = Z + torch.einsum('Nd,d->Nd', H_k, Coef[:,k])

    # Update variables
    second_last_H = last_H
    last_H = H_k

# Transform the final representation into predictions
Y = MLP(ReLU(Z).dropout())
Pred = Softmax(Y)
return Pred
\end{lstlisting}
\end{algorithm}

\subsection{Pseudo-code for OptBasisGNN.}
\label{sec:pseudo_torch_OptBasis}
% # Add noise to reduce the probability of producing zero signals
% Noise = torch.rand_like(X)*1e-5
% X = X + Noise

\begin{algorithm}[H]
\caption{OptBasisGNN.\textit{Pytorch style}.}
\label{alg:cotextde}

\definecolor{codeblue}{rgb}{0.25,0.5,0.5}
\lstset{
  backgroundcolor=\color{white},
  basicstyle=\fontsize{7.2pt}{7.2pt}\ttfamily\selectfont,
  columns=fullflexible,
  breaklines=true,
  captionpos=b,
  commentstyle=\fontsize{7.2pt}{7.2pt}\color{codeblue},
  keywordstyle=\fontsize{7.2pt}{7.2pt},
% #  frame=tb,
}
\begin{lstlisting}[language=python]
# f: raw feature dimension
# d: hidden dimension, or number of channels
# N: number of nodes
# K: order of polynomial basis
# X(Nxd): Input features 
# P(NxN): Sym-normalized adjacency matrix 
# Coef(dxK): coefficient matrix


# Transfer raw input in signals 
X = ReLU(MLP(X.dropout())).dropout()  # (Nxd)

# Normalize H_0
norm = torch.norm(X, dim=0).view(1, d)
norm = torch.clamp(norm, 1e-8)
H_0 = X / norm    # (Nxd)

Z = torch.zeros_like(X)
# Add to the final representation
Z = Z + torch.einsum('Nd,d->Nd', H_0, Coef[:,0])  

last_H = H_0
second_last_H = torch.zeros_like(H_0)

for k in range(1, K):
    H_k = P @ last_H   # (Nxd)

    # Orthogonalize H_k to all the former vectors
    # To achieve this, only 2 substractions are required
    project_1 = torch.einsum('Nd,Nd->1d', H_k, last_H)              # (1xd)
    project_2 = torch.einsum('Nd,Nd->1d', H_k, second_last_H)       # (1xd)
    H_k = H_k - project_1 *  last_H - project_2 * second_last_H     # (Nxd)

    # Normalize H_k
    norm = torch.norm(H_k, dim=0).view(1, d)
    norm = torch.clamp(norm, 1e-8)
    H_k = H_k / norm   # (Nxd)

    # Add to the final representation
    Z = Z + torch.einsum('Nd,d->Nd', H_k, Coef[:,k])

    # Update variables
    second_last_H = last_H
    last_H = H_k

# Transform the final representation to predictions
Y = MLP(ReLU(Z).dropout())
Pred = Softmax(Y)
return Pred
\end{lstlisting}
\end{algorithm}

\section{Experimental Settings.}
% In this subsection, we give more details about the node classification tasks. 

\subsection{Node Classification Tasks on Large and Small Datasets.}
\label{expappendix:nodecls}

\paragraph*{Model setup.} The structure of FavardGNN and OptBasisGNN follow Algorithm~\ref{alg:favardgnn_cls}. 
The hidden size of the first MLP layers $h$ is set to be $64$, 
which is also the number of filter channels. 
For the scaled-up OptBasisGNN, 
we drop the first MLP layer to fix the basis vectors needed for precomputing, and following the scaled-up version of ChebNetII~\cite{he2022chebii}, 
we add a three-layer  MLP with weight matrices of shape 
$F \times h$, $h \times h$ and $h \times c$ 
after the filtering process.

For both models, the initialization of $\alpha$ is set as follows: for each channel $l$, 
the coefficients of the $g_{0,l}$ are set to be $1$, 
while the other coefficients are set as zeros, 
which corresponds to initializing the polynomial filter 
on each channel to be $h(\lambda)=1-\lambda$. 
For the initialization of three-term parameters that determine the initial polynomial bases on each channel, 
we simply set $\{\sqrt{\beta}\}$ to be ones, 
and $\{\gamma\}$ to be zeros. 

\paragraph*{Hyperparameter tunning.} 
For the optimization process on the training sets, we tune all the parameters with Adam ~\cite{kingma2014adam} optimizer. 
We use early stopping with a patience of 300 epochs.

We choose hyperparameters on the validation sets.
To accelerate hyperparameter choosing, 
we use Optuna\cite{akiba2019optuna} to select hyperparameters from the range below with a maximum of 100 complete trials\footnote{We use Optuna's Pruner to drop some hyperparameter choice in an early stay of training. This is called an incomplete/pruned trial.}:
\begin{enumerate}[topsep=0pt,itemsep=-1ex,partopsep=1ex,parsep=1ex]
    \item Truncated Order polynomial series: $K \in \{2, 4, 8, 12, 16, 20\}$; 
    \item Learning rates: $\{0.0005, 0.001, 0.005, 0.1, 0.2, 0.3, 0.4, 0.5\}$;
    \item Weight decays: $ \{1 \mathrm{e}{-8}, \cdots, 1\mathrm{e}{-3} \} $;
    \item Dropout rates: $ \{0., 0.1, \cdots, 0.9 \} $; 
\end{enumerate}

There are two extra hyperparameters for scaled-up OptBasisGNN:
\begin{enumerate}[topsep=0pt,itemsep=-1ex,partopsep=1ex,parsep=1ex]
\item Batch size: $\{ 10,000, 50,000 \}$; 
\item Hidden size (for the post-filtering MLP): $\{ 512, 1024, 2048 \}$.  
\end{enumerate}

\subsection{Multi-Channel Filter Learning Task.}
\label{expappendix:regression}

\paragraph*{YCbCr Channels.}
We put the practical background of our multichannel experiment in 
the YCbCr color space, a useful color space in computer vision and multi-media
~\cite{shaik2015YCbCr}. 

\paragraph*{Our Synthetic Dataset.}
When creating our datasets with 60 samples, 
we use 4 filter combinations on 15 images in \citet{He2021bern}'s single filter learning datasets. The 4 combinations on the three channels are: 
\begin{enumerate}[topsep=0pt,itemsep=-1ex,partopsep=1ex,parsep=1ex]
    \item Band-reject(Y) / low-pass(Cb) / high-pass(Cr); 
    \item High-pass(Y) / High-pass(Cb) / low-pass(Cr);
    \item High-pass(Y) / low-pass(Cb) / High-pass(Cr);
    \item Low-pass(Y) / band-reject(Cb) / band-reject(Cr).
\end{enumerate}

The concrete definitions of the signals, i.e. band-reject are aligned with those given in~\citet{he2022chebii}.

\paragraph*{Visualization on more samples.} We visualize more samples as Figure~\ref{fig:regression} in Figure \ref{fig:more-samples}. In all the samples, the tendencies of different curves are alike. 

% \begin{figure*}
%     \centering
%     \includegraphics[width=0.5\linewidth]{figures/convergence_grid_sample_0_withF.pdf}
%     \caption{Visualization with a sample in multi-channel filter learning task.}
%     \label{fig:sample-0}
%  \end{figure*}

%  \begin{figure*}
%     \centering
%     \includegraphics[width=0.5\linewidth]{figures/convergence_grid_sample_1_withF.pdf}
%     \caption{Visualization with a sample in multi-channel filter learning task.}
%     \label{fig:sample-1}
%  \end{figure*}

%  \begin{figure*}
%     \centering
%     \includegraphics[width=0.5\linewidth]{figures/convergence_grid_sample_2_withF.pdf}
%     \caption{Visualization with a sample in multi-channel filter learning task.}
%     \label{fig:sample-2}
%  \end{figure*}

 \begin{figure}[htp]
  \begin{subfigure}{0.5\textwidth}
    \centering
    \includegraphics[width=1.\textwidth]{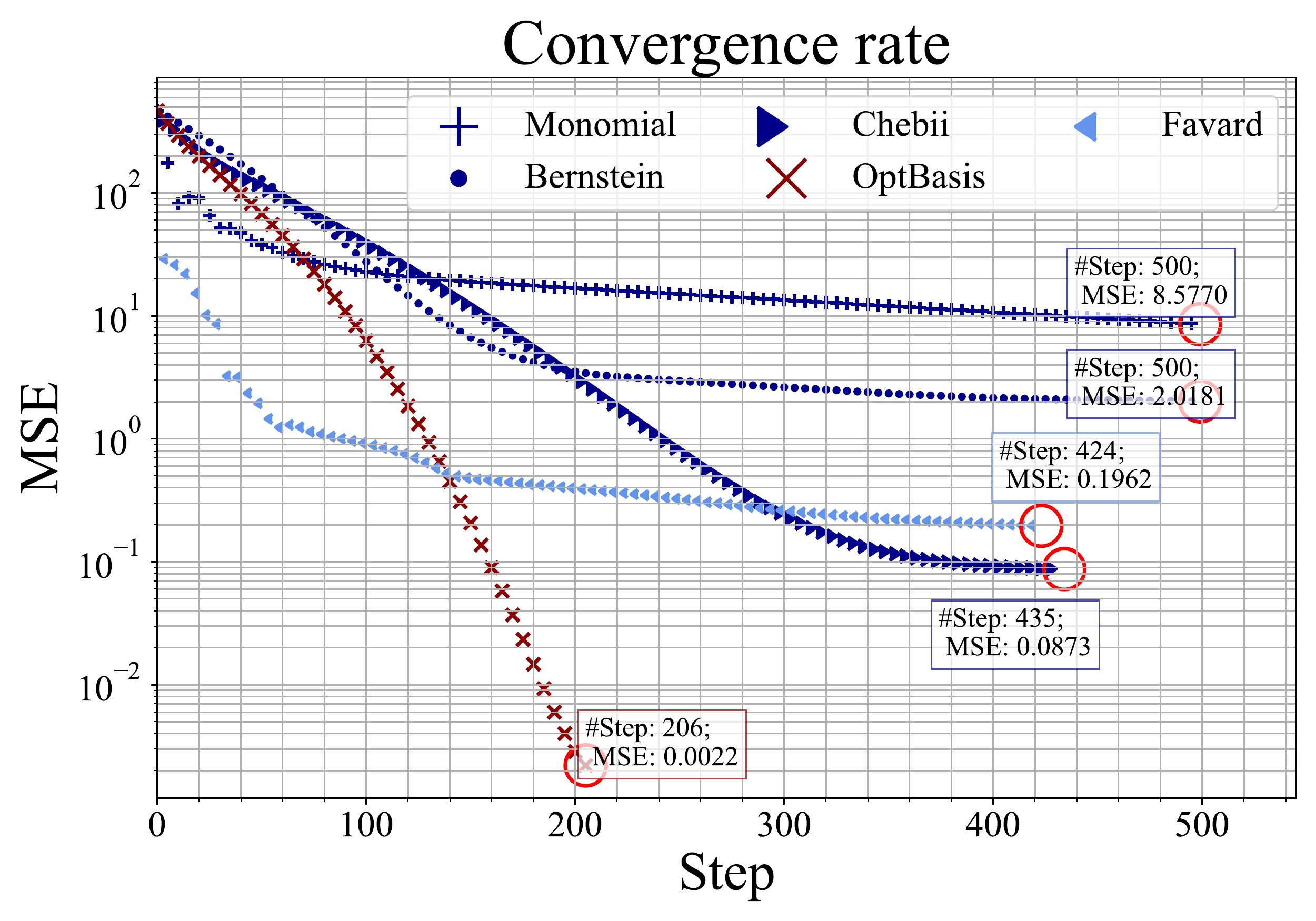}  
  \end{subfigure}
  \begin{subfigure}{0.5\textwidth}
    \centering
    \includegraphics[width=1\textwidth]{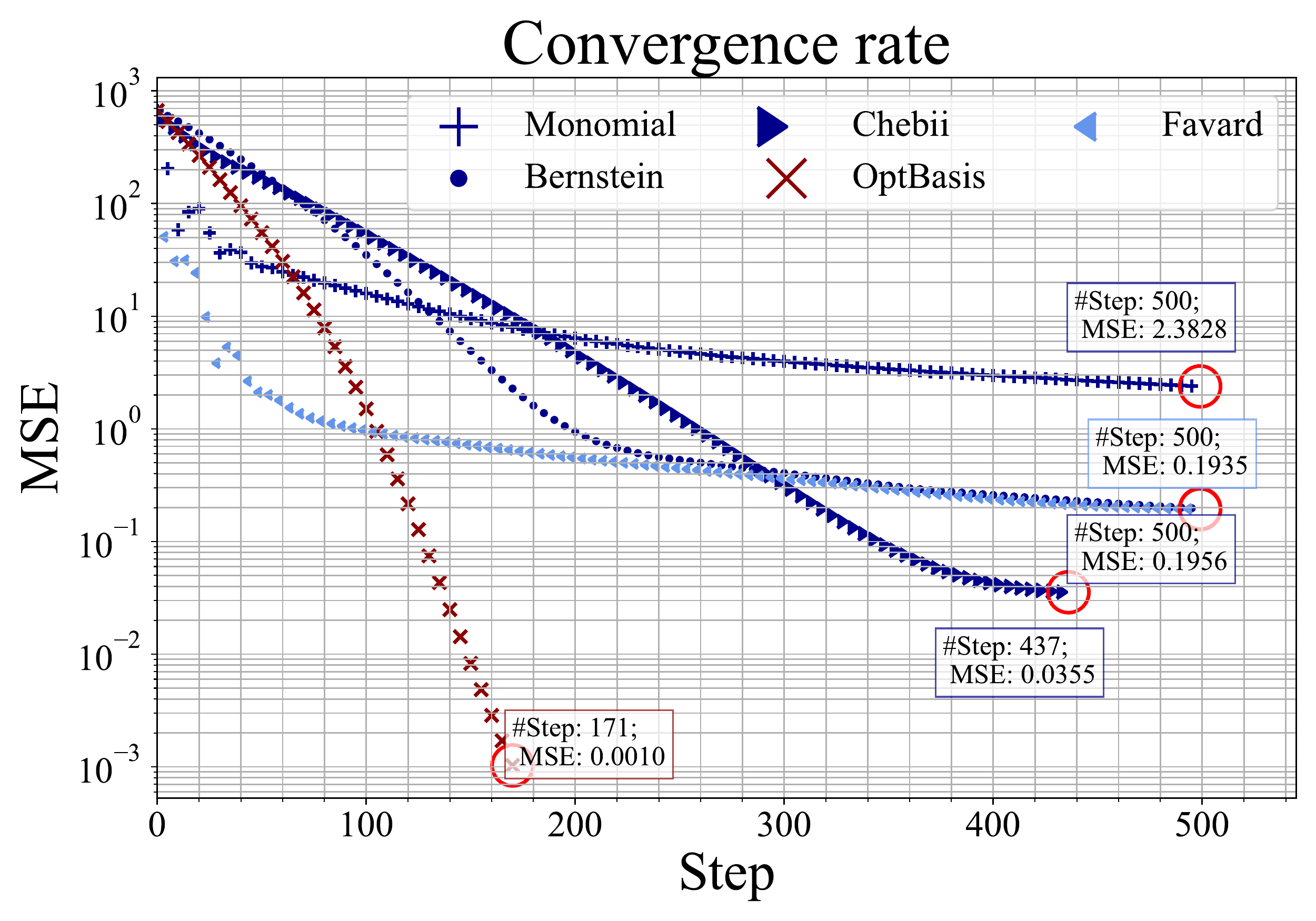}  
  \end{subfigure}
  \begin{subfigure}{0.5\textwidth}
    \centering
    \includegraphics[width=1\textwidth]{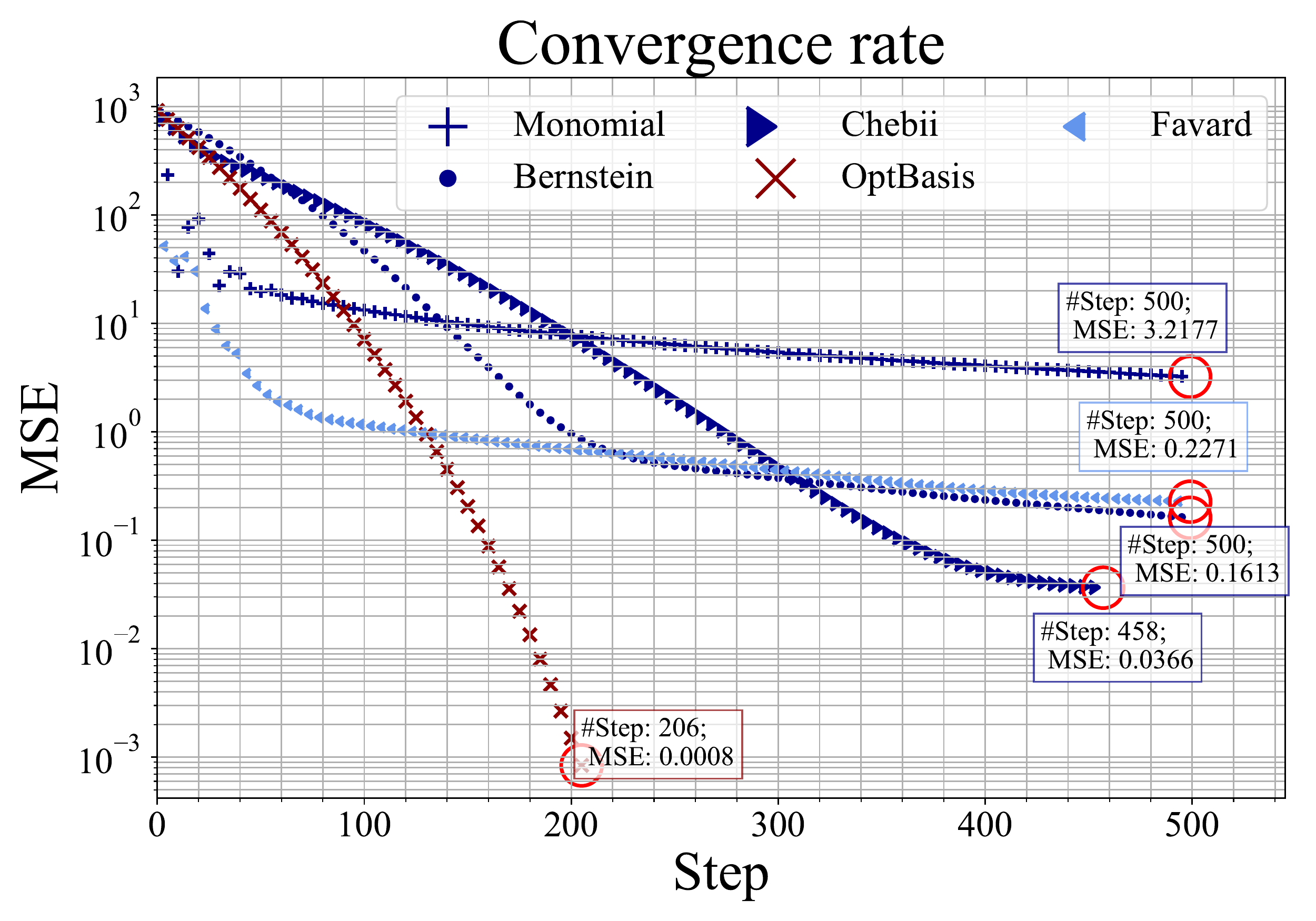}  
  \end{subfigure}
    \begin{subfigure}{0.5\textwidth}
    \centering
    \includegraphics[width=1\textwidth]{figures/convergence_grid_sample_3_withF.pdf}  
  \end{subfigure}
  \caption{Visualization with more samples in the multi-channel filter learning task.
  }
  \label{fig:more-samples}
  \end{figure}

\subsection{Examining of FavardGNN's bump.}
\label{expappendix:bump}
Figure \ref{fig:regression_10000} (right), 
we observe bump with a node classification setup. 
To show this clearer, we let FavardGNN and GPR-GNN (which uses the Monomial basis for classification) to fit \textit{the whole set} of nodes, and move \textrm{dropout} and \textrm{Relu} layers. 
As in the regression re-examine task, 
we cancel the earlystop mechanism,
stretch the epoch number to 10,000, 
and record cross entropy loss on each epoch.

\section{Summary of Wang's work}
\label{sec:SumWang}

This section is a restate for a part of \citet{Wang2022jacobi}. 
% which is crucial for us. 
For the convenience of the reader's reference, 
we write this section here. 
More interested readers are encouraged to refer to the original paper.

\citet{Wang2022jacobi} raise a criterion for best basis, but states that it
{\textbf{cannot be reached}}.

% To a large extend, our work is based on Wang et al.'s. So we will first
% restate Wang et al's progress concretely.

\subsection{The Criterion for Optimal Basis}
\label{sec-Wang}

Following \citet{keyulu2021Optm}, 
\citet{Wang2022jacobi} considers the squared loss 
$R = \frac{1}{2} \| Z - Y \|_\textrm{F}^2$, 
where $Y$ is the target
, and $Z = \underset{l \in [1, h]}{\|}  \sum_{k = 0}^K \alpha_{k, l}
g_{k, l} (\hat{P}) X_{:, l}$ . \footnote{Here, $X$ is not necessarily the raw feature ($X_\textrm{raw}$) but often some thing like $X_\textrm{raw}W$. $W$ is irrelevant to the choice of polynomial basis, 
and merges $W$ into $X$.}

Since each signal channel is independent and contributes independently to the
loss, i.e. $R = \sum_l \frac{1}{2} \| Z_{:, l} - Y_{:, l} \|_\mathrm{F}^2$, we can
then consider the loss function channelwisely and ignore $l$. Loss on one
signal channel $x$ is:
\[ r = \frac{1}{2} \| z - y \|^2_\textrm{F}, \]
where $z = \sum_{k = 0}^K \alpha_k g_k (\hat{P}) x$.

This loss is a convex function w.r.t. $\alpha$. Therefore, the gradient
descents's convergence rate depends on the {\textbf{Hessian matrix}}'s
condition number, denoted as $\kappa (H)$. When $H$ is an identity matrix,
$\kappa (H)$ reaches a minimum and leads to the best convergence rate (Boyd \&
Vandenberghe, 2009).

The Hessian matrix $H$ looks like\footnote{Note that, \citet{Wang2022jacobi}
{\small{}}define $g_{k_2}$ on $\hat{L}$ (or $\{ \lambda_i \}$) while we define
it on $\hat{P}$ (or $\{ \mu_i \}$). They are equivalent.}:
\begin{equation*}
  H_{k_1 k_2} = \frac{\partial^2 r}{\partial \alpha_{k_1} \partial \alpha_{k_2}} = x^\mathrm{T}
  g_{k_2} (\hat{P}) g_{k_1} (\hat{P}) x. 
%   \label{eqHessian}
\end{equation*}
% {\color[HTML]{800000}{\color[HTML]{008000}{\textbf{Definition 1}}
% {\textbf{(Optimal Basis for signal $x$)}}. For a given graph signal $x$,
% polynomial basis $\{ g_k \}_{k = 1}^K$ is optimal in convergence rate when $H$
% given in \eqref{eqHessian} is an identity matrix.}
% {\color[HTML]{800080}{\color[HTML]{008000}}}}

% \begin{tcolorbox}[boxrule=0.pt,height=18mm,valign=center,colback=blue!3!white]
%     \begin{definition}[Optimal Basis for signal $x$]
%         For a given graph signal $x$, polynomial basis $\{ g_k \}_{k = 0}^K$ 
%         is optimal in convergence rate when $H$
%         given in \eqref{eqHessian} is an identity matrix.
%         \end{definition}
% \end{tcolorbox}

\citet{Wang2022jacobi} further write $H_{k_1 k_2}$ in the following form:
\[ H_{k_1 k_2} = x^T g_{k_2} (\hat{P}) g_{k_1} (\hat{P}) x = \sum_{i = 1}^n
   g_{k_1} (\mu_i) g_{k_2} (\mu_i) (U^\mathrm{T} x)_i^2 , 
\]
which can be equivalently expressed as a Riemann sum:
\[ 
\sum^N_{i = 1} g_{k_1} (\mu_i) g_{k_2} (\mu_i) \frac{F (\mu_i) - F (\mu_{i
   - 1})}{\mu_i - \mu_{i - 1}} (\mu_i - \mu_{i - 1}), 
\]
where 
$F (\mu) \assign \sum_{\mu_i \leq \mu} (U^\mathrm{T} x)_i^2$
.
Define 
$f (\mu) = \frac{^{} \vartriangle F (\mu)}{\vartriangle
\mu}$, 
$H_{k_1 k_2}$ comes to
\[ H_{k_1 k_2} = \text{} \int_{\mu = - 1}^1 g_{k_1} (\mu) g_{k_2} (\mu) f
   (\mu) \mathd \mu . 
\]
This suggests that, 
$\{ g_k \}_{k = 0}^{K}$ is an optimal basis when it is
{\textbf{orthonormal}} w.r.t. {\textbf{weight function}} $f (\cdot)$. (For
more about orthonormal basis, see Section~\ref{para:inner_product}.)

\subsection{Wang's Method}

Having write out the weight function $f (\mu)$, 
the optimal basis is determined. 
\citet{Wang2022jacobi} think of a regular process for getting this optimal
basis, which is unreachable since eigendecomposition is unaffordable for
large graphs. 
We summarize this process in Algorithm \ref{alg:unreacheable}.

According to \citet{Wang2022jacobi}, 
the optimal basis would be an orthonormal basis,
but unfortunately, this basis and the exact form of its weight function is
unattainable. 
As a result, 
they come up with a compromise by allowing the
model to choose from the orthogonal Jacobi bases, which have ``{\tmem{flexible
enough weight functions}}'', i.e. $(1 - \mu)^{a} (1 + \mu)^{b}$. The
Jacobi bases are a family of polynomial bases. A specific form Jacobi basis is
determined by two parameters $\left(a,  b\right)$. Similar to the well-known
Chebyshev basis, the Jacobi bases have a recursive formulation, making them
efficient for calculation.

\end{document}